\documentclass[10pt, journal]{IEEEtran}

\usepackage[utf8]{inputenc}

\usepackage{hyperref}
\usepackage{url}
\usepackage{microtype}

\usepackage{graphicx}
\usepackage{booktabs} 
\usepackage[english]{babel}
\usepackage{balance}
\usepackage{multirow}
\usepackage{multicol}
\usepackage[ruled, lined, longend, linesnumbered]{algorithm2e}
\usepackage[font=small,labelfont=bf]{caption}
\usepackage{subcaption}
\usepackage{graphicx}
\usepackage{adjustbox}
\usepackage{amsmath}
\usepackage{amsthm}
\usepackage{amssymb}
\usepackage{xcolor}
\usepackage{dirtytalk}
\usepackage{color, soul}

\SetKwComment{Comment}{/* }{ */}

\usepackage[utf8]{inputenc} 
\usepackage[T1]{fontenc}    
\usepackage{hyperref}       
\usepackage{url}            
\usepackage{booktabs}       
\usepackage{amsfonts}       
\usepackage{nicefrac}       
\usepackage{microtype}      
\usepackage{xcolor}         

\newtheorem{lemma}{Lemma}

\newtheorem{objective}{Objective}

\newcommand{\var}{\mathrm{Var}}

\DeclareMathOperator*{\Expect}{\mathbb{E}}

\hypersetup{
    colorlinks=true,
    linkcolor=cyan,
    citecolor=blue,
}
\makeatletter 
\newcommand{\linebreakand}{%
  \end{@IEEEauthorhalign}
  \hfill\mbox{}\par
  \mbox{}\hfill\begin{@IEEEauthorhalign}
}
\makeatother
\hypersetup{
    colorlinks=true,
    linkcolor=blue,
    citecolor=cyan,
    pdfborder={0 0 0},
}    
\definecolor{duong}{RGB}{0, 0, 0} 
\definecolor{label1}{RGB}{149, 230, 255}
\definecolor{label2}{RGB}{255, 230, 149}

\usepackage{amsmath,amsfonts,bm}

\def\eqref#1{Eq.~(\ref{#1})}

\def\1{\bm{1}}

\DeclareMathAlphabet{\mathsfit}{\encodingdefault}{\sfdefault}{m}{sl}
\SetMathAlphabet{\mathsfit}{bold}{\encodingdefault}{\sfdefault}{bx}{n}

\newcommand{\E}{\mathbb{E}}

\DeclareMathOperator*{\argmin}{arg\,min}

\begin{document}
\title{Knowledge Abstraction for Knowledge-based Semantic Communication: A Generative Causality Invariant Approach}

\author{Minh-Duong Nguyen, Quoc-Viet Pham, \textit{Senior Member, IEEE}, Nguyen~H.~Tran, \textit{Senior Member, IEEE,}, Hoang-Khoi~Do, Duy T. Ngo, \textit{Senior Member, IEEE}, Won-Joo Hwang, \textit{Senior Member, IEEE}

\thanks{Minh-Duong Nguyen is with the Department of Intelligent Computing and Data Science, VinUniversity, Hanoi, Vietnam (e-mail: mduongbkhn@gmail.com).}
\thanks{Hoang-Khoi Do and Quoc-Viet Pham are with the School of Computer Science and Statistics, Trinity College Dublin, Dublin 2 D02PN40, Ireland (e-mail: \{dokh, viet.pham\}@tcd.ie).}
\thanks{Nguyen~H.~Tran is with the School of Computer Science, The University of Sydney, Darlington, NSW 2006, Australia (e-mail: nguyen.tran@sydney.edu.au).}
\thanks{Duy T. Ngo is with the School of Engineering, The University of Newcastle, Callaghan, NSW 2308, Australia (e-mail: duy.ngo@newcastle.edu.au).}
\thanks{Won-Joo Hwang (corresponding authors) is with the Department of Information Convergence Engineering, Pusan National University, Busan 46241, Republic of Korea (email: wjhwang@pusan.ac.kr).}
}

\maketitle
\begin{abstract}
In this study, we design a low-complexity and generalized AI model that can capture common knowledge to improve data reconstruction of the channel decoder for semantic communication.
Specifically, we propose a generative adversarial network that leverages causality-invariant learning to extract causal and non-causal representations from the data. 
Causal representations are invariant and encompass crucial information to identify the data's label.
They can encapsulate semantic knowledge and facilitate effective data reconstruction at the receiver.
Moreover, the causal mechanism ensures that learned representations remain consistent across different domains, making the system reliable even with users collecting data from diverse domains.
As user-collected data evolves over time causing knowledge divergence among users, we design sparse update protocols to improve the invariant properties of the knowledge while minimizing communication overheads.
Three key observations were drawn from our empirical evaluations. Firstly, causality-invariant knowledge ensures consistency across different devices despite the diverse training data. Secondly, invariant knowledge has promising performance in classification tasks, which is pivotal for goal-oriented semantic communications. Thirdly, our knowledge-based data reconstruction highlights the robustness of our decoder, which surpasses other state-of-the-art data reconstruction and semantic compression methods in terms of Peak Signal-to-Noise Ratio (PSNR).
\end{abstract}

\section{Introduction}
The interconnected nature of the modern world enables innovative services, such as virtual reality, mobile immersive extended reality, and autonomous driving. These advancements introduce significant challenges for communication systems, such as resource scarcity, network traffic congestion, and the need for scalable connectivity to support edge intelligence.
Semantic Communications (SemCom) has recently been considered a promising technology \cite{sun2025edge, liu2025multi, 2024-SemCom-Survey2} to address these issues. In contrast to conventional communications, SemCom aims to convey the meaning of information by selectively extracting and transmitting only pertinent information relevant to recipients' tasks. Consequently, one of the most crucial missions for SemCom is improving the efficiency of data compression and reconstruction. 

\textbf{Related Works and Challenges}: Existing works in SemCom can be categorized by the types of information encoding and decoding. One category focuses on improving the efficiency of semantic encoders and decoders. This approach emphasizes data compression while ensuring sufficient information is retained for accurate reconstruction at the receiver's end \cite{2023-FL-HCFL, 2024-SemCom-DRGO, yu2025multi}. For example, in DeepSC \cite{2021-SEM-DeepSC}, a transformer-based architecture was designed to improve data compression and reconstruction. U-DeepSC \cite{2023-SemCom-UDeepSC} and MUDeepSC \cite{2022-SemCom-MUDeepSC} were developed as upgraded versions, in which layer-wise knowledge transfer technique is utilized to improve information exchange across tasks, facilitating the extraction of semantic information from both texts and images. 
Mem-DeepSC \cite{2023-SemCom-MemDeepSC} deploys the memory on the receiver to further improve the data reconstruction of DeepSC.

The studies presented in \cite{2022-SEM-AdaptableSemanticCompression, 2020-SemCom-DJSCCF, 2019-SemCOm-DJSCC-WIT, 2024-SemCom-AdaSem} introduced adaptable semantic compression methods for a semantic encoder-decoder system which possesses the ability to assess latent representations. By doing so, data features that have the least impact on predicting the task can be eliminated. 
SemCC \cite{2024-SemCom-SemCC} employed contrastive learning to extract salient features from data, optimizing these features for transmission over wireless channels.
DeepMA \cite{2024-SemCom-DeepMA} extracted orthogonal semantic symbol vectors from the data to achieve orthogonal multiple access in SemCom systems.
GenerativeJSCC \cite{2023-SemCom-GenerativeJSCC} integrates StyleGAN-2 and employs a strategy that leverages pre-trained AI models to enhance data reconstruction efficiency. Despite its robustness, GenerativeJSCC performs effectively only when the dataset aligns with the pre-trained model, limiting its ability to generalize to new tasks or domains.

Another category in SemCom is knowledge-based SemCom \cite{2024-SemCom-Survey2}. 
Currently, there are three popular approaches to designing knowledge-aided semantic communication (SemCom): (1) using channel feedback as common knowledge, (2) leveraging prior knowledge, and (3) employing knowledge extractors, either through deterministic algorithms or knowledge graph generators.
First, the methods utilizing channel feedback as prior knowledge \cite{2024-SemCom-AdaSem, 2024-SemCom-SCAN}, enable the encoder and decoder to collaborate simultaneously to achieve efficient data reconstruction. \cite{2023-SemCom-DeepJSCCV} introduced OraNet, a network designed to predict the quality of reconstructed images by jointly considering channel SNR, compression rate, and semantic code.
\cite{2023-SemCom-VLSCC} introduced VL-SCC which is adaptable to various SemCom systems, with the rate-allocation scheme learned end-to-end using proxy functions.
From the perspective of prior knowledge, \cite{2023-SemCom-TaskUnaware} leveraged the generative Artificial Intelligence (AI) model to generate pseudo data to support the data reconstruction at the receivers. This generative model requires additional knowledge designed by humans, which is \emph{time consuming and expensive}. 
JCM \cite{2024-SemCom-JCM} assumes that a shared dataset is available as common knowledge for both the encoding and decoding processes. However, this approach is inherently reliant on the specific dataset, making it \emph{difficult to generalize} to other datasets or applications.
From the perspective of knowledge extractor, current approaches often depend on task-specific algorithms (e.g., Natural Language Processing (NLP) \cite{2022-KG-LogicalQueryAnswering, 2022-KG-RelationExtraction, 2022-SemCom-AdaptiveBitRate}), or employ complex Large Language Models (LLMs) \cite{2023-IKR-T2IPR, 2023-IKR-CIRCRN, 2023-IKR-Reveal}). These approaches have \emph{limited applicability across tasks} and are unsuitable for mobile communication systems due to \emph{their high computational complexity and resource consumption}.

\textbf{Motivations}:
The current challenges in designing semantic knowledge stem from distinguishing data characteristics across various clients, which is known as domain shifts \cite{2022-DG-CIRL}. 
Traditional methods for achieving invariant representations often involve regularization to minimize domain gaps \cite{2022-DG-KLGuided, 2021-DG-IRL-DDT}. However, this strategy is impractical in wireless communications due to the communication overhead associated with the necessity of data sharing across the domains (i.e., clients). To address these issues, we consider structured causal model (SCM) \cite{2022-DG-InvariantRationale}, in which invariance is stable across domains. 
With this mechanism, we utilize causality invariance learning to extract invariant representations for semantic knowledge. 
\textbf{Contributions}:
In this paper, we propose Unified Knowledge retrieval via Invariant Extractor (UKIE), algorithm for extracting semantic knowledge. UKIE employs two primary loss functions for knowledge extraction: invariant learning loss and variant learning loss. The invariant learning loss ensures the representations are causally invariant. In contrast, the variant learning loss leverages data transmission to capture client-specific features.
Drawing inspiration from recent advancements in deep invariant representation encoding methods \cite{2020-DG-EntropyReg, 2018-DG-CIAN, 2021-DG-ExploitDomainSpecific} which utilize generative adversarial networks \cite{2014-ML-GAN} (GAN), we develop a GAN-based architecture and two discriminative losses to learn causality invariant representation. The first loss aids invariant representations in achieving classification accuracy, while the second adversarial loss helps to filter out irrelevant information. This GAN-based architecture enables us to focus on embedding the most meaningful information into the invariant representations without memorizing the knowledge into the discriminative classifier's parameters. 
In essence, our contributions comprise:
\begin{itemize}
    \item We introduce UKIE, a GAN-based causality-invariance learning architecture, to extract knowledge that is invariant across the data with the same label.
    \item We propose a novel architecture for knowledge-based SemCom that guarantees: 1) the extracted invariant knowledge is utilized as semantic knowledge to improve the data reconstruction efficiency at the receiver, and 2) sparse update protocols enable the capability of stable learning for invariant representations as the data from different clients evolve over time.
    \item We conduct a series of numerical experiments to evaluate the effectiveness of retrieving invariant knowledge and to compare the performance of UKIE with existing semantic compression approaches. Additionally, we perform numerical experiments on datasets with domain shifts to demonstrate the robustness of our invariant knowledge when applied to clients whose data exhibits significant feature variation. UKIE consistently outperforms baseline approaches, achieving over a $10\%$ improvement in both Peak Signal-to-Noise Ratio (PSNR) for data reconstruction and test accuracy in goal-oriented semantic communication tasks.
\end{itemize}
The remainder of this paper is organized as follows. Section~\ref{sec:preliminaries} presents the preliminaries and foundational techniques used in this work. Section~\ref{sec:system-model} describes the system model and overall architecture our proposed method. Section~\ref{sec:ukie-training} details the training procedure of the proposed UKIE framework. Section~\ref{sec:experimental-eval} provides experimental evaluations to assess the effectiveness of our approach and compare it with existing baselines. Finally, Section~\ref{sec:conclusion} concludes the paper.
\section{Backgrounds \& Preliminaries}\label{sec:preliminaries}
\subsection{Knowledge-aided Semantic Communication}
{\color{black}
Knowledge-aided SemCom extends traditional semantic communication by incorporating a shared knowledge base between the transmitter and receiver. Communication in this framework relies on either a pre-established or dynamically updated knowledge base that facilitates more efficient and meaningful information exchange. This shared knowledge may include ontologies (i.e., formal definitions of domain-specific concepts and their interrelations), knowledge graphs (i.e., structured representations of entities \cite{2022-SemCom-KG-Cognitive}, facts, and their relationships), or learned models (i.e., neural networks trained on domain data, capture semantic patterns useful for interpretation and reasoning), which enable context-aware encoding, transmission, and decoding of semantic content. By leveraging the shared knowledge, knowledge-aided SemCom enhances semantic alignment and enables communication systems to convey intent and meaning more effectively.

Currently, knowledge is integrated into SemCom systems through several distinct approaches. Firstly, human-defined knowledge or prompts can be employed as inputs to LLMs or text/image encoders to generate semantic representations. These knowledge-informed representations can then be incorporated into the latent features at the decoder side, thereby enriching the semantic understanding during message reconstruction. Secondly, knowledge can serve as a form of channel feedback from the receiver, capturing characteristics of communication environment. This feedback can be leveraged to enhance the encoding and decoding processes by adapting to dynamic channel conditions. Lastly, knowledge can be directly introduced as a shared input to both the encoder and decoder, ensuring that both ends of the communication system possess a common understanding, which facilitates more efficient and accurate semantic transmission.
}

\subsection{Conditional Generative AI}
{\color{black}
The foundational work of conditional Generative AI (GAI) originates from the Conditional Variational Autoencoder (CVAE)~\cite{2015-DL-CVAE}. The conditional GAI models a joint generative process comprising two main components. The first is an encoding process that learns a latent representation $z$ from the input data $x$, and is represented by the conditional distribution $p(z \vert x, u)$. The second is a decoding process that reconstructs the data $x$ from the latent variable $z$, modeled by the conditional distribution $p(x \vert z, u)$. The principal distinction between conditional GAI and its classical counterpart lies in the explicit use of conditioning variables $u$, which are drawn from a predefined set $\mathcal{U}$. These variables serve to guide the generation process more effectively than unconditioned models. Building upon the foundation laid by CVAE, conditional generative AI has since been widely adopted across a range of applications, yielding improved performance in comparison to classical generative models. Notable extensions include applications on CVAE \cite{CVAE-3DHuman, CVAE-Audio2Gestures, CVAE-ManifoldDimension}, Conditional Generative Adversarial Network (CGAN)~\cite{CGAN-I2I, CGAN-ImageGeneration, CGAN-ImageSynthesis} and the Guided Diffusion Model~\cite{CDM-ImageRecon, CDM-VideoSynthesis}. The incorporation of conditioning variables $c$ can further reduce the overall loss, as they can effectively substitute for certain active latent dimensions \cite{CVAE-ManifoldDimension}. 
}

{\color{black}
\subsection{Prototypical Networks}\label{sec:protonet}
Prototypical Networks (ProtoNet) \cite{snell2017prototypical} are a class of metric-based few-shot learning models introduced to address the challenge of learning from limited labeled examples. The core idea of ProtoNet is to represent each class by a prototype, which is typically the mean vector of embedded support examples for that class in a learned feature space. 
Let $f_{\theta_{K}}: \mathcal{X} \rightarrow \mathbb{R}^{d_{z_K}}$ be an embedding function parameterized by ${\theta_{K}}$, which maps an input $x \in \mathcal{X}$ to a $d_{z_K}$-dimensional vector in the feature space. This function is used to encode both support and query examples.
For each class $c$ in the support set $S$, the prototype $z_K$ is computed as the mean of the embedded support examples:
\begin{align}
    z^c_K = \frac{1}{|S_c|} \sum_{(x_i, y_i) \in S_c} f_{\theta_{K}}(x_i),
\end{align}
where $S_c$ denotes the subset of the support set with label $c$.
For a query sample $x_q$, classification is based on the distance between its embedding $f_{\theta_{K}}(x_i)$ and each class prototype $z^c_K$. ProtoNet is trained as follows: 
\begin{align}
    {\theta^*_{K}} = \argmin_{{\theta_{K}}} -\frac{\exp(\textrm{sim}(f_{\theta_{K}}(x), z^c_K))}{\sum_{c'} \exp(\textrm{sim}(f_{\theta_{K}}(x), z^{c'}_K))},
\end{align}
where $\textrm{sim}(\cdot, \cdot)$ represents the cosine similarity between two representations.
The training objective encourages intra-class compactness and inter-class separability. Embeddings of instances from the same class are pulled closer to their corresponding prototype, while embeddings from different classes are pushed farther apart. This promotes both the precision of prototype estimation and the discriminability of the learned representations.
}

\section{Knowledge-aided Semantic Communication}\label{sec:system-model}
\textcolor{duong}{We present a knowledge-aided SemCom framework, comprising} four modules, namely \textbf{1)} data-knowledge extractor, \textbf{2)} physical channel, \textbf{3)} SemCom system with invariant knowledge, and \textbf{4)} semantic knowledge-aided data decoder. {\color{black}The encoder and decoder utilize the robustness of conditional GAI by incorporating knowledge as a conditional variable, thereby enhancing the encoding and decoding processes.}
\subsection{System Model}\label{sec:problem-formulation}
Consider a network with $J$ users in which the users transmit data over a wireless network. Let $x\in\mathbb{R}^{d_x}$ denote the source data. Each data point $x$ has a corresponding label $y\in\mathbb{R}$. In our study, we assume that the $J$ users communicate with others using the SemCom framework (see Fig.~\ref{fig:UKIE-application}), which includes a physical channel for data transmission, and a semantic channel for semantic knowledge sharing among devices. 
\begin{figure*}[!ht]
\centering
\includegraphics[width = 0.9\linewidth]{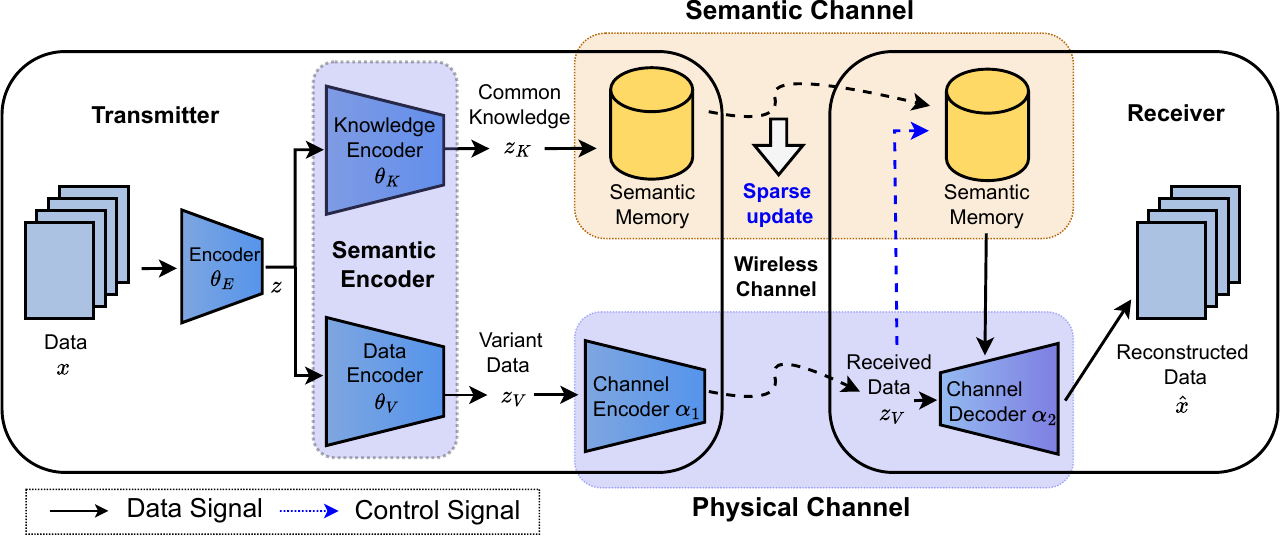}
\caption{Application of UKIE in SemCom.}
\label{fig:UKIE-application}
\end{figure*}
\subsection{Data-Knowledge Extractor}\label{sec:knowledge-extractor}
Apart from the conventional source encoder, semantic encoder $\theta$ is capable of extracting semantic knowledge from the data, supporting data reconstruction at the receiver's end.
The parameters of the encoder are denoted as $\theta_1 = \{\theta_{E}, \theta_{K}, \theta_{V}\}$, where $\theta_{E}, \theta_{K}, \theta_{V}$ represent the representation, semantic knowledge, and variant data encoders, respectively. $\theta_{V}$ and $\theta_{K}$ are the two primary blocks for extracting knowledge and data, while the purpose of $\theta_{E}$ is to reduce redundant information (e.g., background) in data by applying the information bottleneck technique. The semantic encoder $\theta_1$ has two purposes. Firstly, it encodes the original data into semantic knowledge using the \emph{variant data encoder} $f_{\theta_{V}}: \mathbb{R}^{d_z}\rightarrow\mathbb{R}^{d_{z_V}}$:
\begin{align}
    z_V = f_{\theta_{V}}(z),\quad \textrm{s.t.} \quad z = f_{\theta_{E}}(x),
\end{align}
where $z_V \in \mathbb{R}^{d_{z_V}}$ are referred to as physical data and are transmitted over the wireless network. $z$ is the latent representation encoded by the encoder $\theta_{E}$. The representation encoder is characterized by $f_{\theta_{E}}: \mathbb{R}^{d_x}\rightarrow\mathbb{R}^{d_z}$.
Secondly, it extracts and constructs knowledge of the dataset via \emph{semantic knowledge encoder} $f_{\theta_{K}}: \mathbb{R}^{d_x}\rightarrow\mathbb{R}^{d_{z_K}}$. For instance, 
\begin{align}
    z_K = f_{\theta_{K}}(z),\quad \textrm{s.t.} \quad z = f_{\theta_{E}}(x),
\end{align}
where $z_K\in \mathbb{R}^{d_{z_K}}$ denotes the knowledge extracted by the \textit{semantic encoder}. The semantic knowledge $z_K$ is stored in the semantic memory of the distributed devices to facilitate the data reconstruction at the decoder. 
Based on the aforementioned purposes, the \textit{semantic encoder} has the following objectives in knowledge and data extraction:
\begin{objective}[Invariant knowledge]
The semantic knowledge $z_K$ always remains \emph{invariant with respect to the label}. 
\label{obj:inv-knowledge}
\end{objective}

\begin{objective}[Efficient compression]
The physical data $z_V$ has a smaller size than the original data $x$ because the semantic knowledge $z_K$ is excluded from the original data $x$.
\label{obj:efficient-compression}
\end{objective} 

\begin{objective}[Efficient knowledge extraction]
The physical data $z_V$ exhibits \emph{variance among samples in one class}. 
\label{obj:efficient-knowledge}
\end{objective}

\begin{objective}[Meaningful representation]
The semantic knowledge and data contain meaningful information to support the capability of reconstructing original data.     
\label{obj:meaningful-representation}
\end{objective}

Objective~\ref{obj:inv-knowledge} ensures that the semantic knowledge remains stable across distributed devices, whose data is sampled from distinguished domains with different data distributions \cite{2023-DG-FDG1, 2023-FL-FDG2}.
Furthermore, by finding invariant with respect to the category label, the label index can be leveraged for choosing the appropriate semantic knowledge for data reconstruction at the receiver's end without carrying a significant amount of knowledge via the physical channel (which is demonstrated in Section~\ref{sec:knowledge-aided-compression}).
Owing to this established consistency, the semantic channel, which is tasked with transmitting semantic knowledge, does not need frequent transmissions. As a consequence, Objectives~\ref{obj:inv-knowledge} and \ref{obj:efficient-compression} ensure that communication overheads of the SemCom system are significantly reduced.
Objective~\ref{obj:efficient-knowledge} is required due to:
\begin{lemma}
The data representations, in which invariant knowledge is excluded, have higher variance than those in the data representations that include invariant knowledge.
\label{lemma:data-variant-property}
\end{lemma}
\begin{proof}
    \begin{align}
        \var(z) = \frac{1}{Nd_z}\sum^{N}_{i=1}\Big\Vert z^{i} - z^{i}_K\Big\Vert^2,
    \end{align}
    where $N$ is the number of data, and $z_K$ is the invariant knowledge. As the UKIE extracts the data into two type of components $z_K$ and $z_V$, we have: 
    \begin{align}
        \var(z) 
        &= \frac{1}{Nd_z}\sum^{N}_{i=1} \Bigg[
          \Big\Vert z^{i}_V - z^{i}_K\Big\Vert^2 
        + \Big\Vert z^{i}_K - z^{i}_K\Big\Vert^2\Bigg]
        \\
        &= \frac{d_{z_V}}{Nd_z}\sum^{N}_{i=1} \Bigg[
          \frac{1}{d_{z_V}}\Big\Vert z^{i}_V - z^{i}_K\Big\Vert^2\Bigg]
         = \frac{d_{z_V}}{d_z}\var(z_V). \notag
    \end{align}
    Therefore, we have $\var(z_V)\geq \var(z)$.
\end{proof}

Objective~\ref{obj:meaningful-representation} ensures the aggregated information of invariant knowledge $z_K$ and data $z_V$ is more meaningful. Thus, the receiver can reconstruct the data effectively.
Given the extracted knowledge, the semantic knowledge is stored in the semantic storage. This storage plays a crucial role in both preserving semantic knowledge and overseeing updates through the semantic channel (which is demonstrated in Section~\ref{sec:semantic-channel}).
\subsection{Physical Channel}
In the physical channel, lossy compression is first applied to the transmitted data via the \emph{channel encoder} $f_{\alpha_1}:\mathbb{R}^{d_{z_V}}\rightarrow\mathbb{C}^{M\times1}$ with the parameter set $\alpha_1$. Therefore, we have the data compression process as follows:
\begin{align}
    s = f_{\alpha_1}(z_V).
\label{eq:encoder}
\end{align}
{\color{duong}For transmission over a wireless network with Rayleigh fading channels,} the received data $\hat{s}$ can be represented as 
\begin{align}
    \hat{s} = h\cdot s + n,
\end{align}
where $\hat{s} \in \mathbb{C}^{M\times 1}$ {\color{duong}consists of $M$ symbols, $h$ represents the channel state, which is assumed to be perfectly estimated using some pilot signals.} $n \sim \mathcal{CN}(0, \sigma^2_n \mathbf{I})$ indicates independent and identically distributed (i.i.d) Gaussian noise, $\sigma^2_n$ is the noise variance for each user's channel and $\mathbf{I}$ is the identity matrix. In addition to the primary data, the transmitted data $s$ also carries the label index $c$ of the corresponding source data $x$. Consequently, the decoder can select appropriate semantic knowledge from the semantic memory to facilitate data reconstruction (see Section~\ref{sec:knowledge-aided-compression}).
\subsection{Semantic Channel}\label{sec:semantic-channel}
As mentioned previously, semantic knowledge $\{z^c_K \vert c=1,\ldots, C\}$, where $C$ is the number of labels of the source data $x$, is stored in the semantic storage, which is installed in every distributed user. Whenever there are changes in semantic storage over a certain threshold, users will broadcast semantic knowledge to other distributed users. To this end, the semantic knowledge on user $j$ ($i,j\in[J], i\neq j$) receives sparse updates from user $i$ over the semantic channel as follows:
\begin{align}
z_{K}(j,t) =
\begin{cases}
z_{K}(j,t), \quad\text{if}~\Vert\Delta z_{K}(i,t)\Vert \leq \kappa, \\
z_{K}(j,t) + \Delta z_{K}(i,t) , \quad\text{otherwise},
\end{cases}
\label{eq:semantic-update}
\end{align}
where $\Delta z_{K}(i,t) = z_{K}(i,t) - z_{K}(i,t-\tau)$ represents the changes in the invariant representations on user $i$ over time step $t$ after $\tau$ seconds. We impose a threshold constraint $\kappa$ to regulate the sparse update frequency of the semantic channel. Because of the invariant characteristics of the semantic knowledge as mentioned in Section~\ref{sec:knowledge-extractor}, the semantic knowledge sustain low statistical noise as time evolves. Therefore, frequent updates are not required as shown in \eqref{eq:semantic-update}. Notably, choosing a smaller value for $\kappa$ allows data reconstruction to be performed more precisely when the distributed devices adapt to a new data distribution. Conversely, by selecting a larger value for $\kappa$, communication costs can be reduced at the expense of slower decoder adaptation owing to less frequent updates.
\subsection{Knowledge-aided Lossy Compression}\label{sec:knowledge-aided-compression}
In Section~\ref{sec:knowledge-extractor}, we assume that the knowledge remains invariant concerning the labels. To select suitable semantic knowledge for the semantic decoder, we employ the label index associated with the transmitted data. After data has been received, a label index is provided to semantic memory, allowing retrieval of relevant invariant representations to facilitate efficient data reconstruction at the semantic decoder.
Specifically, the appropriate semantic knowledge $z_k$ to the data $\hat{s}$ is chosen from semantic memory based on the label as $z_{K} = z^{c}_{K}$, where $c$ denotes the label index.
Given the received data $\hat{s}$ and semantic knowledge $z_K$, the data $\hat{x}\in \mathbb{R}^{d_x}$ is reconstructed as follows:
\begin{align}
    \hat{x} = g_{\theta_2}(\hat{z}_{{V}}, z_K) 
    =
    g_{\theta_2}(g_{\alpha_2}(\hat{s}), z_K),
\label{eq:decoder}
\end{align}
where $\alpha_2$, $\theta_2$ denote the parameters of the \textit{channel decoder} and \textit{semantic decoder}, respectively. We have $g_{\alpha_2}:\mathbb{C}^{M\times1}\rightarrow\mathbb{R}^{d_{z_V}}$ and $g_{\theta_2}:\mathbb{R}^{d_{z_V}+d_{z_K}}\rightarrow \mathbb{R}^{d_{x}}$. 
\begin{algorithm}[t]
\caption{Pseudo algorithm for training UKIE.}
\label{alg:UKIE}
\textbf{Input:} Initial model parameter $\theta_0$, learning coefficients $\alpha_\textrm{rec},\alpha_\textrm{v},\alpha_\textrm{iv},\alpha_\textrm{gtc},\alpha_\textrm{adv} \in [0,1)$ and learning rate $\eta_\textrm{UKIE},\eta_\textrm{MID},\eta_\textrm{adv} \in\mathbb{R}^{+}$.\\
\For{training round $r = 0,1,\dots,I $}{
    /*\textit{ UKIE Generator }*/ \\
    \For{iteration $i = 0,1,\dots,I_\textrm{UKIE}-1 $}{
        Compute loss for the UKIE model:
        \begin{align}
            \mathcal{L}_{\textrm{UKIE}} = \alpha_\textrm{rec}\mathcal{L}_{\textrm{rec}} + \alpha_\textrm{iv}\mathcal{L}_{\textrm{iv}} + 
            \alpha_\textrm{v}\mathcal{L}_{\textrm{v}}. \notag
        \end{align} \\ 
        Apply gradient descent according to~\eqref{eq:UKIE-update}.
    }   
    /*\textit{ Meaningful Invariant Discriminator }*/ \\
    \For{iteration $ i = I_\textrm{UKIE},\dots,I_\textrm{UKIE}+I_\textrm{MID} $}{
        Compute loss for meaningful invariant representation:
        \begin{align}
            \mathcal{L}_{\textrm{MID}} =  \alpha_\textrm{iv}\mathcal{L}_{\textrm{iv}} + 
            \alpha_\textrm{gtc}\mathcal{L}_{\textrm{gtc}}. \notag
        \end{align} \\ 
        Compute adversarial loss $\mathcal{L}_\textrm{adv}$. \\
        Apply gradient descent according to~\eqref{eq:MID-update}, \eqref{eq:ADV-update}.
    }   
}
\end{algorithm}
In the next section, we present the architecture and learning process of our joint knowledge encoder - knowledge-aided data reconstructor, which generates knowledge that is invariant with respect to the data label. 
\section{Knowledge Encoder for Semantic Communications}\label{sec:ukie-training}
In this section, we develop UKIE, a GAN-based model that identifies the optimal invariant and variant representations, and guarantees that the representations have meaningful information for data reconstruction. UKIE separates information into two distinct components, i.e., invariant representations $z_K$ and variant representations $z_V$. Irrelevant information is then filtered out from these representations. UKIE operates as an \emph{invariant encoder} that shares the same goal as the \emph{knowledge encoder} $\theta_K$. To effectively train UKIE, the following criteria must be met: 1) two representations $z_K, z_V$ are invariant and variant, respectively, 2) both $z_K, z_V$ comprehensively retain all the relevant information from the original data. The overall framework is shown in Fig.~\ref{fig:GAN-UKIE}. 

To meet these criteria, we first design the loss components and joint optimization problem that support the training of UKIE (demonstrated in Section~\ref{sec:variance-invariance-ke}). For efficient UKIE learning, we disentangle the joint optimization problem into two interleaved learning processes (demonstrated in Section~\ref{sec:two-stage-optimization}).
\begin{figure}[!ht]
\centering
\includegraphics[width = 0.85\linewidth]{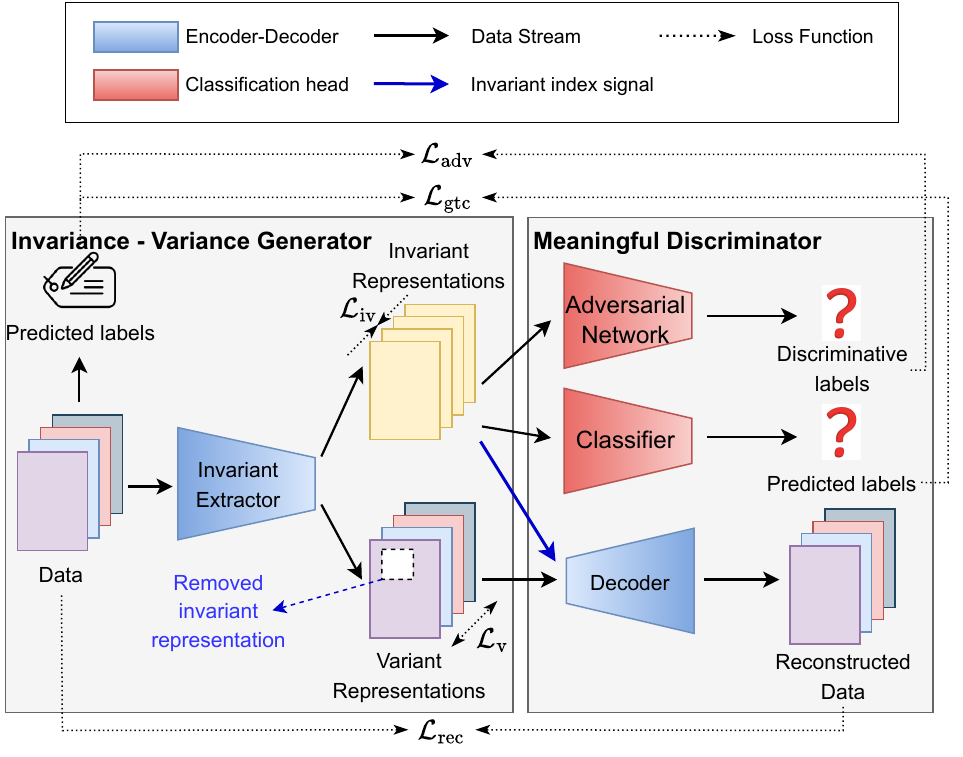}
\caption{The framework of UKIE. The main system block of UKIE is the \emph{invariant encoder} $\theta_1$, which can be considered the \emph{invariance-varariance generator}. The invariant encoder has the objective of learning the two distinguished representations, i.e., invariant and variant representations. To train UKIE, we use the \emph{Meaningful Discriminator} $\{\theta_{\textrm{gtc}}, \theta_{\textrm{rec}}\}$ to anchor the extracted representations to retain the meaningful information from the original data. The adversarial network $\theta_{\textrm{adv}}$ is used to find the anomalies in the invariant representations so that the invariant extractor can improve itself. Five loss functions are used in the framework. $\mathcal{L}_{\textrm{iv}}, \mathcal{L}_{\textrm{v}}$ make the two extracted representations $z_K, z_V$ invariant and variant, respectively. The three discriminator losses $\mathcal{L}_{\textrm{rec}}, \mathcal{L}_{\textrm{gtc}}, \mathcal{L}_{\textrm{adv}}$ are losses in the \emph{Meaningful Discriminator}.}
\label{fig:GAN-UKIE}
\end{figure}
\subsection{Variance-Invariance as a Knowledge Encoder}\label{sec:variance-invariance-ke}
\subsubsection{Invariant Representations Learning}
To guarantee that the extracted features remain \emph{invariant} to their corresponding label (Objective~\ref{obj:inv-knowledge}), it is imperative to obtain encoded representations with minimal empirical distance to other representations within the same class. When training an \emph{invariant} knowledge encoder, we employ the training process as mentioned in Section~\ref{sec:protonet}. Therefore, we have the following:
\begin{align}
    \mathcal{L}_{\textrm{iv}} = -\sum^C_{c=1} \frac{\E_{x\in B_c}[\exp(\textrm{sim}(f_{\theta_{K}}(x), z^c_K)])}{\sum_{c'} \E_{x\in B_{c'}}[\exp(\textrm{sim}(f_{\theta_{K}}(x), z^{c'}_K))]},
    \label{eq:UKIE-invariant}
\end{align}
where $B_c, B_{c'}$ represent the data in batch $B$ that belongs to the label $c, c'$, respectively. The loss is computed across $C$ label instances in the dataset. We abuse the notation $z^{(i)}_K$ to represent the encoded representations from the original data with the index $i$.
\subsubsection{Variant Representations Learning}\label{sec:variant-learning}
UKIE also extracts variant representations for data transmission via the physical channel (i.e., Objective~\ref{obj:efficient-knowledge}). To extract variant representations effectively, we maximize the empirical distance among representations within a specific label. However, two problems may arise during the maximization of the MSE between samples:
\begin{enumerate}
    \item When the MSE value is at its initial value (at which the variance is relatively low, i.e., around $0$), the optimization process prioritizes tasks with high initial losses. This prioritization can overshadow the learning of variant representations in the early stages.
    \item When the MSE value is already optimized, performing the maximization process further can increase the variance, potentially diminishing the learning of other tasks.
\end{enumerate}
To mitigate these problems, we adopt the hinge loss \cite{2022-IL-VICReg} in designing the variant loss $\mathcal{L}_{\textrm{v}}$:
\begin{align}
    \mathcal{L}_{\textrm{v}} = \frac{1}{C}\sum^C_{c=1}\mathcal{L}^c_{\textrm{v}} = \frac{1}{C}\sum^C_{c=1} \max\Big(0, 1 - \sqrt{\var(z^{ c}_{{V}})+\epsilon}\Big),
    \label{eq:UKIE-variant}
\end{align}
where $\var(z^{c}_{{V}})$ represents the variance among samples in the same class $c$, and
\begin{align}
    \var(z^{ c}_{{V}}) = \mathbb{E}_{i \in  B_c} \left[\Big\Vert z^{(i)}_{{V}} - \mathbb{E}_{j \in  B_c}\Big(z^{(j)}_{{V}}\Big)\Big\Vert^2\right].
\end{align}
Incorporating the hinge loss brings the variant learning function into better alignment with other loss functions used in UKIE. This alignment involves initiating the optimization from higher values to mitigate trivial losses during the initial phase, while setting a lower boundary at zero to prevent biased learning towards the variant loss $\mathcal{L}_{\textrm{v}}$ as the learning progresses. 

\subsubsection{Meaningful Representations}
When decomposing the data into two distinct components (i.e., variants $z^{(i)}_{V}$ and invariants $z^{(i)}_{K}$), it is essential that the disentangled representations retain sufficient meaningful information to facilitate accurate data reconstruction, as dictated by the SCM. Hence, our goal for training UKIE is two-fold: 
\begin{itemize}
    \item The invariant representations capture crucial information that accurately represents their respective label.
    \item The two extracted components can be combined and collectively used to reconstruct the original data with the highest possible accuracy.
\end{itemize}
To achieve the first goal, we follow the SCM \cite{2016-Causal-Primer}. Specifically,
it is possible to predict a specific label from its invariant representations.
To achieve this, we introduce the label classification function $\mathcal{L}_{\textrm{gtc}}$ as follows:
\begin{align}
    \mathcal{L}_{\textrm{gtc}} = -\mathbb{E}_{i\in B} \left[P(y^{(i)}) \log P\Big(f_{\xi}(z^{(i)}_{K})\Big)\right],
    \label{eq:UKIE-invariant-meaning}
\end{align}
where $f_{\xi}:\mathbb{R}^{d_{z_K}}\rightarrow\mathbb{R}$ is the classifier network, parameterized by $\xi$, handling the classification task for the invariant representations $z^{(i)}_{K}$. $y^{(i)}$ is the label of data instance $i$. $P(\cdot)$ denotes the probability distributions.

However, UKIE cannot extract causality invariant representations if it is solely optimized by the classification loss. The reason is that crucial information may be memorized in the classifier $\xi$, which causes overfitting when operating in new users \cite{2020-DG-EntropyReg}. This problem can be alleviated by exploiting the adversarial learning \cite{2018-DG-CIAN}. Specifically, we introduce an adversarial discriminator $\Psi$, and train the classifier and discriminator as a min-max fairness game as follows: 
\begin{align}
    \min_{\theta_K, \theta_E}&\max_\Psi \mathcal{L}_\textrm{adv} 
    = \underset{i\in  B}{\mathbb{E}} \left[P(y^{(i)}) \log\frac{1}{P\Big[f_{\Psi}(z^{(i)}_{K})\Big]}\right], \label{eq:adversarial}\\
    &~\textrm{s.t.} ~~~z^{(i)}_{K} =f_{\theta_{K}}(f_{\theta_{E}}(x^{(i)})), \notag
\end{align}
where $f_{\Psi}:\mathbb{R}^{d_{z_K}}\rightarrow\mathbb{R}$. The purpose of~\eqref{eq:adversarial} is 
to stimulate knowledge generated by UKIE to deceive the adversarial discriminator $\Psi$, whereas the discriminator $\Psi$ endeavors to identify any false information in the invariant knowledge $z_K$.

To achieve the second goal, we employ the reconstruction loss $\mathcal{L}_{\textrm{rec}}$. Specifically, we minimize the MSE between the original data and its reconstructed version. To counteract the risk of overfitting, where the Deep Neural Network (DNN) might encode essential information solely in its model parameters and hinder in generalization across different domains, we employ the Variation AutoEncoder (VAE) loss \cite{2013-DL-VAE}. More precisely, the loss components are as follows:
\begin{align}
		\mathcal{L}_{\textrm{rec}}
		= \sum^C_{c=1}\Big[\mathcal{L}^c_{\textrm{MSE}} - \mathcal{L}^c_{\textrm{KL}}\Big].
    \label{eq:UKIE-Rec}
\end{align}
Here, $\mathcal{L}_{\textrm{rec}}$ computes the joint loss MSE and KL overall label $c\in C$. The label-wise loss $\mathcal{L}^c_{\textrm{MSE}}$ and $\mathcal{L}^c_{\textrm{KL}}$ are defined as:
\begin{align}
    \mathcal{L}^c_{\textrm{MSE}} &= \Expect_{x\sim X^c} \Big[\Big(x - g_{\theta_2}(f_{\theta_1}(x))\Big)^2\Big],
    \\
    \mathcal{L}^c_{\textrm{KL}} &= \Expect_{x\sim X^c}\Big[P(z\vert x) \log [P(z\vert x) / P(z\vert \hat{x})]\Big)\Big].
\end{align}
Here, $\mathcal{L}^c_{\textrm{MSE}}$ denotes the MSE loss between the original data $x$ and its reconstructed version obtained through the inference of the encoder $f_{\theta_1}$ and decoder $g_{\theta_2}$. Minimizing the $\mathcal{L}^c_{\textrm{MSE}}$ term ensures that the reconstructed data closely resembles the original data. The $\mathcal{L}^c_{\textrm{KL}}$ term represents the KL divergence loss, quantifying the statistical distance between two conditional probabilities: the inference from the original data to the latent representations $P(z\vert x)$ and that from the latent representations to the reconstructed data $P(z\vert \hat{x})$. By maximizing the KL divergence $\mathcal{L}^c_{\textrm{KL}}$, we enhance decoder generalization by compelling the encoder to output a distribution over the latent space instead of a single point. By combining these two terms, we can have both the invariant and variant representations \emph{capturing meaningful information} from the source data.

\subsubsection{The impact of SCM to the meaningful representations of UKIE}\label{sec:impact}
To evaluate the robustness of using SCM for knowledge extraction, particularly in extracting meaningful invariant and variant representations for data reconstruction, we propose the following lemma:
\begin{lemma}
Given the disentangled representations $z_K$, and $z_V$ from the original representations $z$, the meaningful information (characterized by entropy $H(\cdot)$) on the aggregated representation $z_K\cup z_V$, brings the most crucial information to the data reconstruction task when the two representations are independent, i.e., $I(z_K; z_V)=0$. 
\label{lemma:meaningful-representations}
\end{lemma}
\begin{proof}
    The aggregated representation at the receiver's end can be considered as $z_K\cup z_V$. Thus, we have the following information measurement on the aggregated representation: 
    \begin{align}
        H(z_K\cup z_V) 
        &= H(z_K) + H(z_V) - I(z_K;z_V) \notag\\
        &\overset{(a)}{\leq} H(z_K) + H(z_V).
    \end{align}
    Here, the equality $(a)$ holds when $I(z_K;z_V)=0$, which means the two representations are independent.
\end{proof}

According to \cite{2022-IL-DIR}, the two disentangled representations by the SCM approach are almost independent of each other. As a result, the decoder at the receiver's end can access the most essential information for data reconstruction, thereby enhancing the robustness of UKIE in data reconstruction.

\subsection{Two-stage Optimization of UKIE} \label{sec:two-stage-optimization}
As mentioned in Section~\ref{sec:variance-invariance-ke}, the joint loss function is complex which makes the optimization problem NP-hard. Acknowledging this shortcoming, we decompose the UKIE training process into two steps, denoted as the \emph{UKIE Generator} and the \emph{Meaningful Invariant Discriminator} (MID). The \emph{UKIE Generator} is responsible for training the encoder to achieve good data reconstruction while maintaining the invariant and variant characteristics of $z_K$ and $z_V$, respectively. The \emph{MID} focuses on distilling essential knowledge into the invariant representation $z_K$. By splitting the complicated tasks into two simpler tasks with fewer loss components, the individual tasks become significantly more straightforward, thereby enhancing the overall training process. To train these two disentangled tasks jointly and efficiently, we use a cost-effective meta-learning technique called Reptile \cite{2018-MeL-Reptile}. Reptile is proven to be computationally efficient while maintaining performance comparable to conventional meta-learning approaches \cite{2017-Mel-MAML}.

In our Reptile-based UKIE, the \emph{UKIE Generator} and \emph{MID} are trained alternately in a continuous manner. At the start of each training round $r$, we begin with the \emph{UKIE Generator} to prioritize the training of meaningful representations that can be faithfully reconstructed. Specifically, the UKIE model is updated within the \emph{UKIE Generator} as follows:
\begin{align}
    &\theta^{r,i+1}_{E} = \theta^{r,i}_{E} - \eta_\textrm{UKIE}\nabla_{\theta^{r,i}_{E}} \mathcal{L}_{\textrm{UKIE}}, \notag \\
    &\theta^{r,i+1}_{V} = \theta^{r,i}_{V} - \eta_\textrm{UKIE}\nabla_{\theta^{r,i}_{V}} \mathcal{L}_{\textrm{UKIE}}, \notag \\
    &\theta^{r,i+1}_{K} = \theta^{r,i}_{K} - \eta_\textrm{UKIE}\nabla_{\theta^{r,i}_{K}} \mathcal{L}_{\textrm{UKIE}}, \notag \\
    &\theta^{r,i+1}_2 = \theta^{r,i}_2 - \eta_\textrm{UKIE}\nabla_{\theta^{r,i}_2} \mathcal{L}_{\textrm{UKIE}}. 
\label{eq:UKIE-update}
\end{align}
In the subsequent \emph{MID}, we focus on the training of encoders $\theta_E, \theta_K$ and finetune the label classifier $\theta_\textrm{gtc}$ as follows: 
\begin{align}
    &\theta^{r,i+1}_{E} = \theta^{r,i}_{E} - \eta_\textrm{MID}\nabla_{\theta^{r,i}_{V}} \mathcal{L}^{r,i}_{\textrm{MID}}. \notag \\
    &\xi^{r,i+1} = \xi^{r,i} - \eta_\textrm{MID}\nabla_{\theta^{r,i}_\mathrm{gtc}} \mathcal{L}^{r,i}_{\textrm{MID}},
\label{eq:MID-update}
\end{align}   
and the adversarial update is applied as follows: 
\begin{align}
    &\Psi^{r,i+1} = \Psi^{r,i} + \eta_\textrm{adv}\nabla_{\Psi^{r,i}} \mathcal{L}^{r,i}_{\textrm{adv}}, \notag \\ 
    &\theta^{r,i+1}_E = \theta^{r,i}_E - \eta_\textrm{adv}\nabla_{\theta^{r,i}_E} \mathcal{L}^{r,i}_{\textrm{adv}}, \notag \\
    &\theta^{r,i+1}_K = \theta^{r,i}_K - \eta_\textrm{adv}\nabla_{\theta^{r,i}_K} \mathcal{L}^{r,i}_{\textrm{adv}}.  
\label{eq:ADV-update}
\end{align}
The details of the training process are shown in Alg.~\ref{alg:UKIE}.

\section{Experimental Evaluation}
\label{sec:experimental-eval}
\subsection{Experimental Setting}\label{sec:sec-sett}
\label{sec:experimental-setting}
Our implementation is based on PyTorch. The experiments are conducted on Ubuntu 22.04.4 LTS with a 13th Gen Intel(R) Core(TM) i9-13900HX processor, 32GB RAM and a GeForce RTX 4070 GPU. \emph{The reproducible repository will be released upon acceptance.}
\subsubsection{Dataset and metrics}
We conduct experiments with four popular datasets, comprising MNIST \cite{2010-Data-MNIST}, EMNIST \cite{2017-Data-EMNIST}, CIFAR-10 \cite{2009-Data-CIFAR}, and CINIC-10 \cite{2018-Data-CINIC10}. To evaluate the efficiency of invariant knowledge abstraction, we use the variance to measure the homogeneity of semantic knowledge, and the variant representations for the data transmission. We evaluate physical data reconstruction using mean squared error (MSE) and PSNR to measure the fidelity of reconstructed images relative to the originals.
\subsubsection{Model settings}
We employ ResNet-9 \cite{2016-DL-Resnet} for both the \textit{semantic encoder} and \textit{semantic decoder}. For the invariant and variant representation extractors, we use a shallow CNN with 3 layers. Additionally, we utilize a shallow fully-connected DNN with 3 layers to serve as the naive classifier and adversarial discriminator for the invariant representations. In all our experiment evaluations, we use a batch size of 128.
\subsubsection{Baselines}
{\color{duong}In our experiments, UKIE is compared with other methods in SemCom, including DeepSC \cite{2021-SEM-DeepSC} and DJSCC-N \cite{2020-SemCom-DJSCCF}, SemCC \cite{2024-SemCom-SemCC}, and DeepMA \cite{2024-SemCom-DeepMA}. 
We further compare the performance of the proposed approaches with conventional digital communication techniques that utilize separate source and channel coding under an equivalent compression ratio. For source coding, we employ the Better Portable Graphics (BPG) image compression algorithm, which is derived from the intra-frame encoding methodology of the High-Efficiency Video Coding (H.265) standard. For channel coding, we adopt LDPC codes configured according to the IEEE 802.16E standard, with a block length of 2304 and coding rates of $1/2$, $1/3$, and $3/4$ in our simulations. QAM with modulation orders of $4$, $16$, and $64$ is utilized. For clarity, we report results corresponding only to the optimal combination of LDPC rates and modulation schemes.
Additionally, we evaluate the upper-bound performance of the digital communication approach, referred to as BPG+Capacity. This upper bound reflects capacity-achieving transmission, as defined by Shannon's theorem, for a given transmit signal-to-noise ratio (SNR) under the assumption of error-free transmission. Consequently, no practical digital transmission scheme that incorporates channel coding and modulation can surpass this theoretical upper bound.}

\begin{figure*}[!ht]
     \centering
     \begin{subfigure}[b]{0.24\textwidth}
         \centering
         \includegraphics[width=\textwidth]{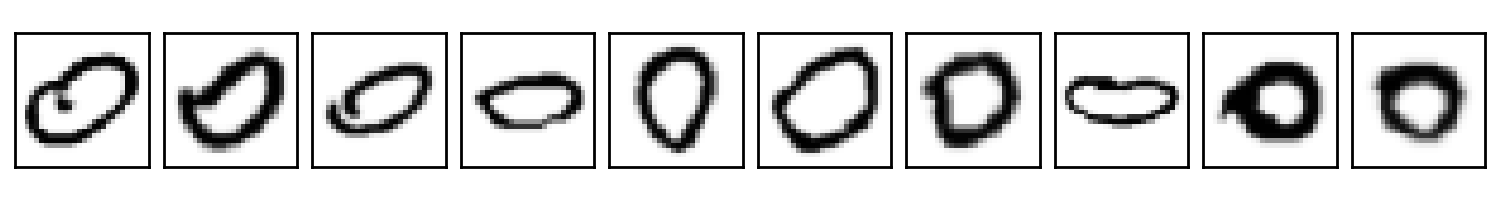}
     \end{subfigure}
     \begin{subfigure}[b]{0.24\textwidth}
         \centering
         \includegraphics[width=\textwidth]{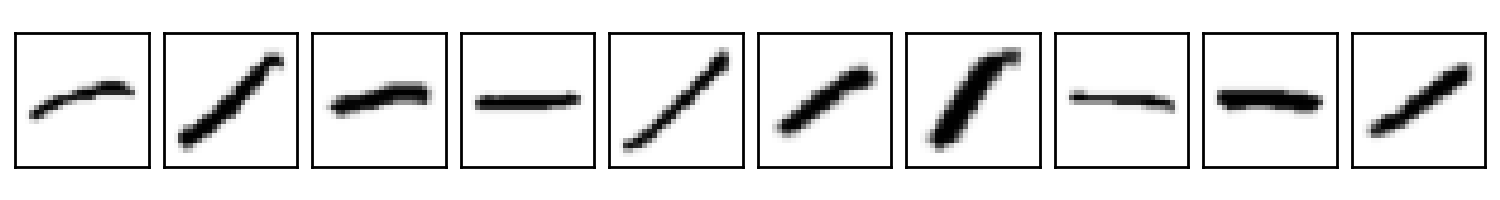}
     \end{subfigure}
     \begin{subfigure}[b]{0.24\textwidth}
         \centering
         \includegraphics[width=\textwidth]{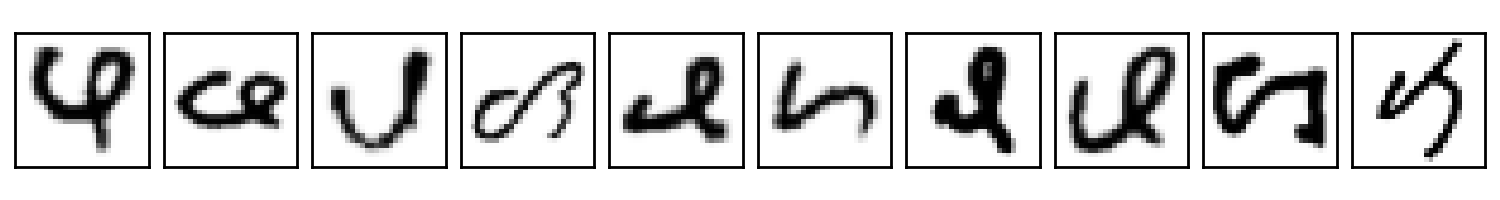}
     \end{subfigure}
     \begin{subfigure}[b]{0.24\textwidth}
         \centering
         \includegraphics[width=\textwidth]{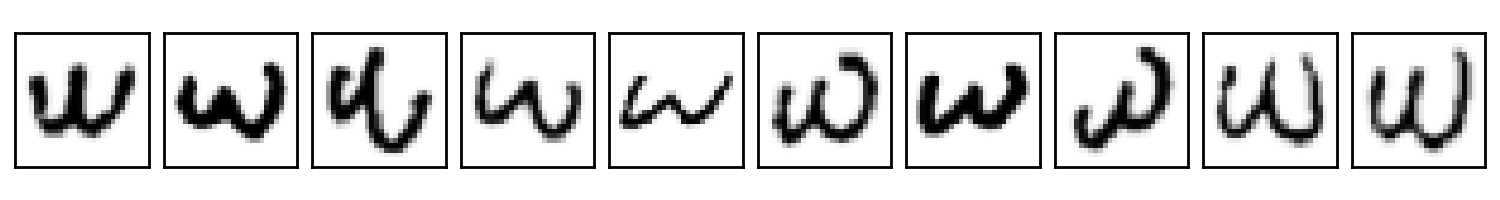}
     \end{subfigure} 
     \\
     \begin{subfigure}[b]{0.24\textwidth}
         \centering
         \includegraphics[width=\textwidth]{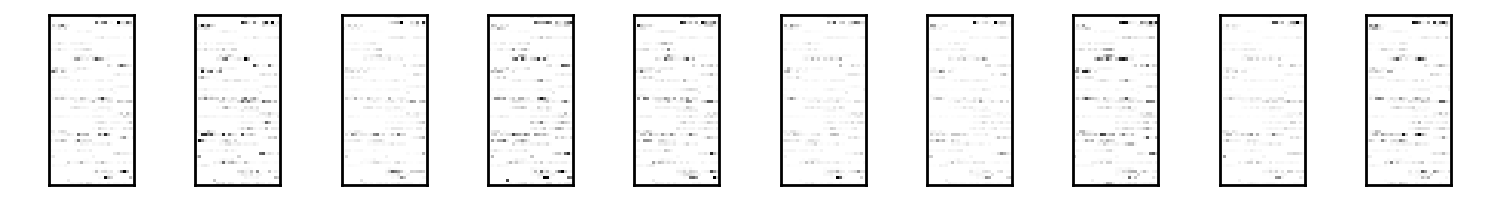}
     \end{subfigure}
     \begin{subfigure}[b]{0.24\textwidth}
         \centering
         \includegraphics[width=\textwidth]{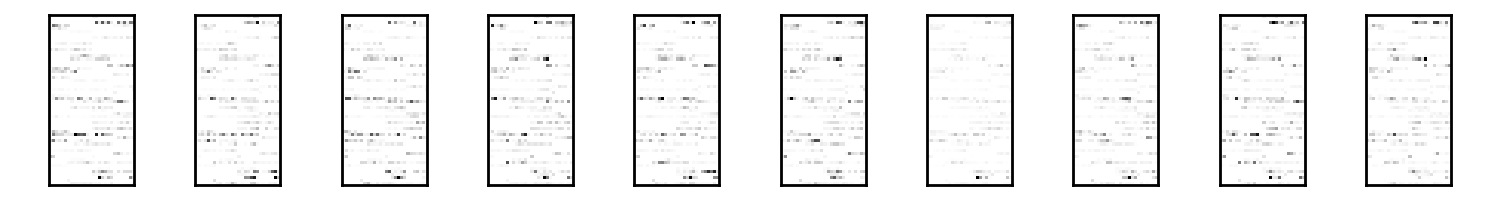}
     \end{subfigure}
     \begin{subfigure}[b]{0.24\textwidth}
         \centering
         \includegraphics[width=\textwidth]{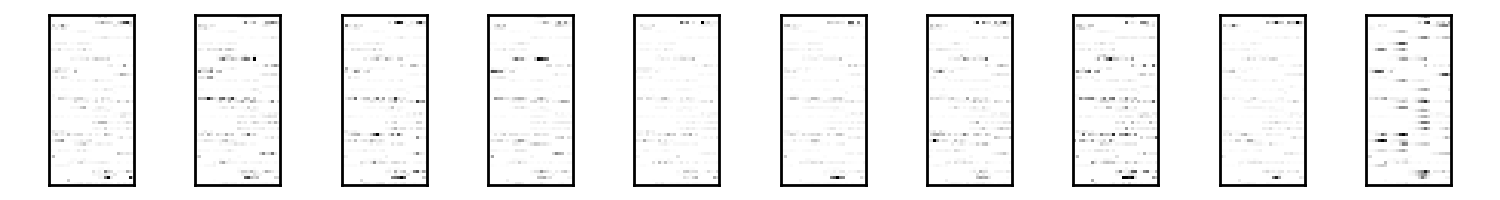}
     \end{subfigure}
     \begin{subfigure}[b]{0.24\textwidth}
         \centering
         \includegraphics[width=\textwidth]{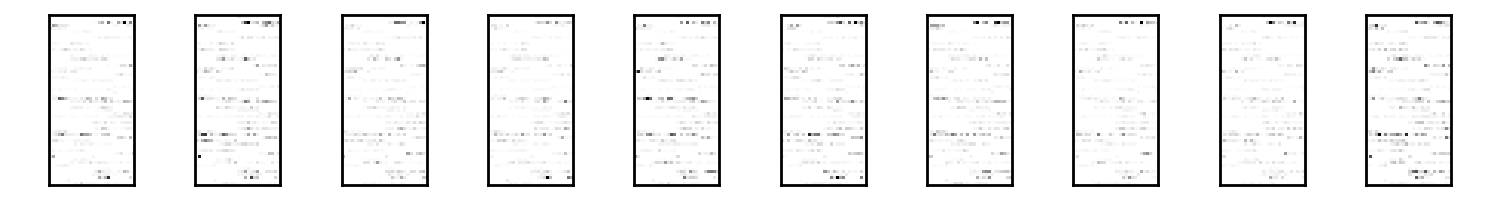}
     \end{subfigure} 
     \\
     \begin{subfigure}[b]{0.24\textwidth}
         \centering
         \includegraphics[width=\textwidth]{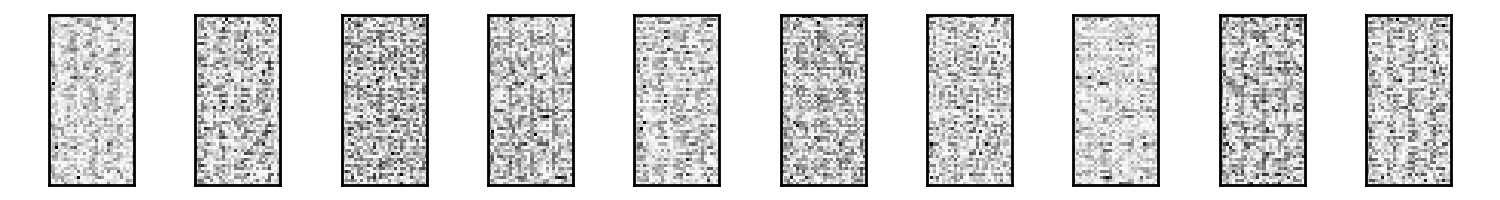}
     \end{subfigure}
     \begin{subfigure}[b]{0.24\textwidth}
         \centering
         \includegraphics[width=\textwidth]{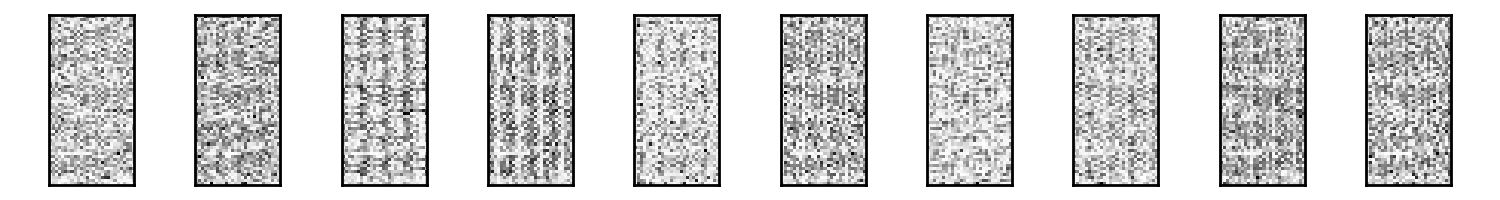}
     \end{subfigure}
     \begin{subfigure}[b]{0.24\textwidth}
         \centering
         \includegraphics[width=\textwidth]{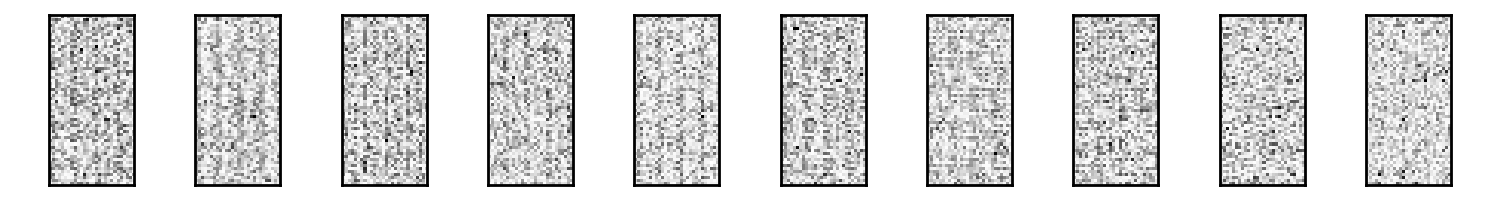}
     \end{subfigure}
     \begin{subfigure}[b]{0.24\textwidth}
         \centering
         \includegraphics[width=\textwidth]{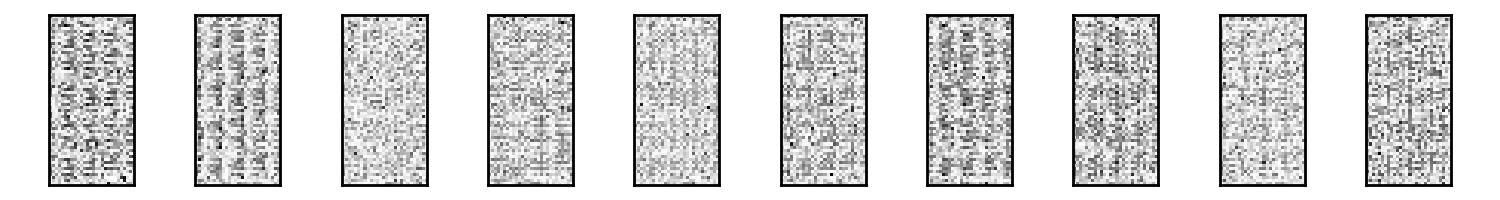}
     \end{subfigure} 
     \\
     \begin{subfigure}[b]{0.24\textwidth}
         \centering
         \includegraphics[width=\textwidth]{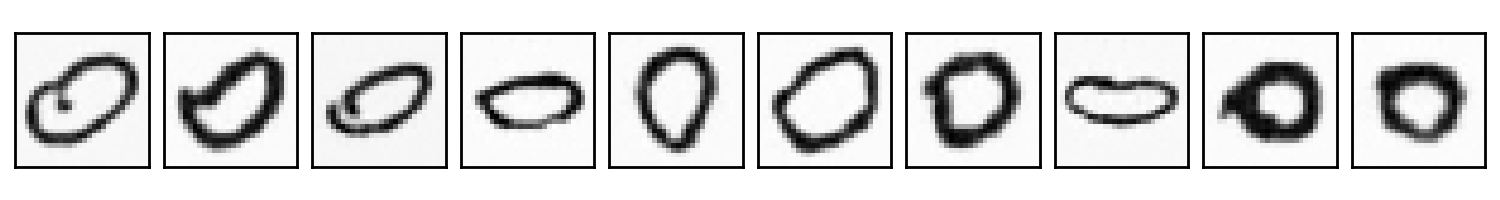}
     \end{subfigure}
     \begin{subfigure}[b]{0.24\textwidth}
         \centering
         \includegraphics[width=\textwidth]{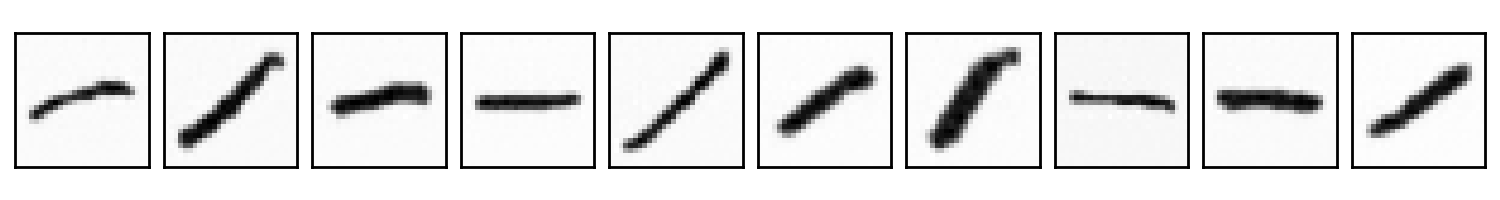}
     \end{subfigure}
     \begin{subfigure}[b]{0.24\textwidth}
         \centering
         \includegraphics[width=\textwidth]{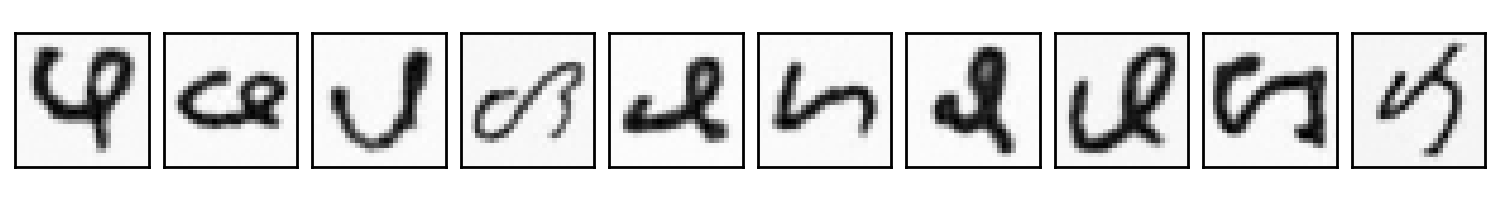}
     \end{subfigure}
     \begin{subfigure}[b]{0.24\textwidth}
         \centering
         \includegraphics[width=\textwidth]{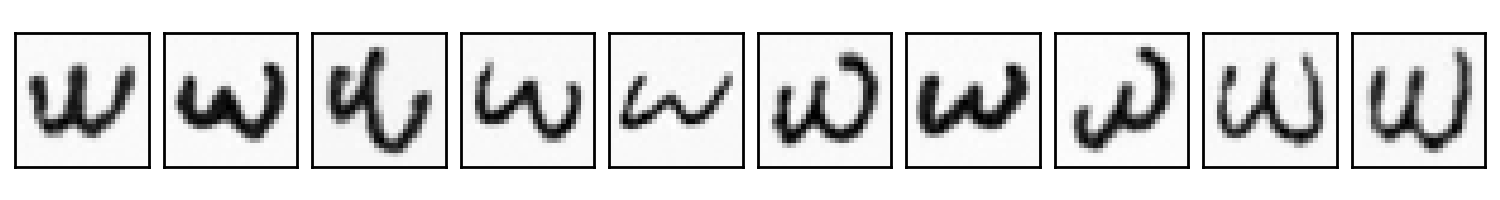}
     \end{subfigure} 
     \caption{The figure shows the data extraction and data reconstruction in EMNIST dataset. The first row shows the original data, the second row displays the invariant representations, the third row displays the variant representations, and the last row shows the reconstructed data.}
     \label{fig:UKIE-EMNIST}
\end{figure*}

\begin{figure*}[!ht]
     \centering
     \begin{subfigure}[b]{0.26\textwidth}
         \centering
         \includegraphics[width=\textwidth]{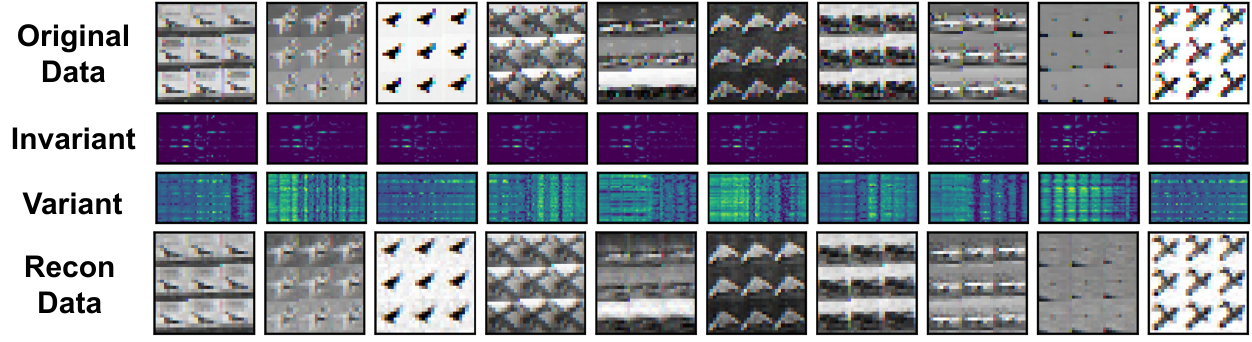}
     \end{subfigure}
     \begin{subfigure}[b]{0.23\textwidth}
         \centering
         \includegraphics[width=\textwidth]{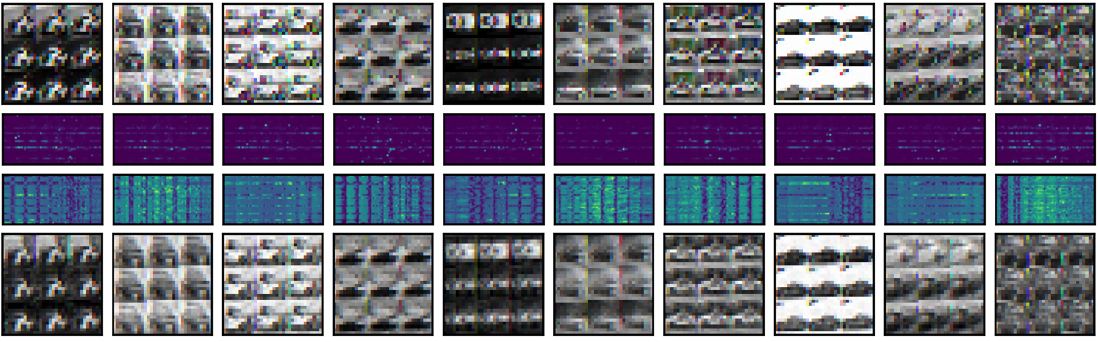}
     \end{subfigure}
     \begin{subfigure}[b]{0.23\textwidth}
         \centering
         \includegraphics[width=\textwidth]{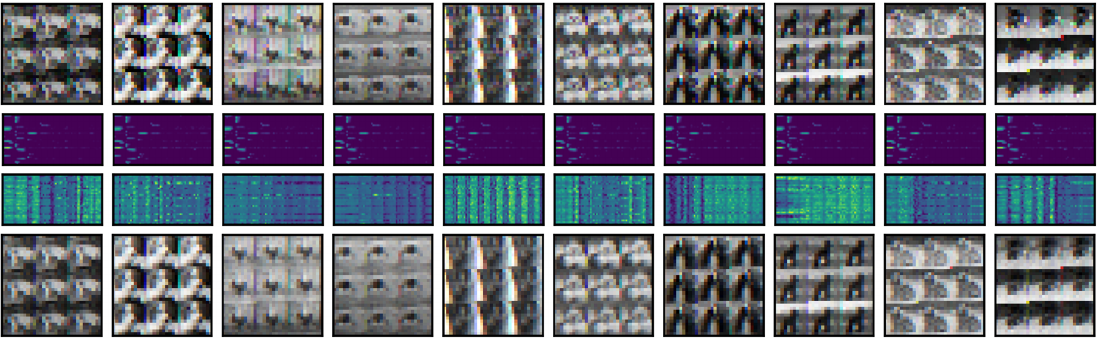}
     \end{subfigure}
     \begin{subfigure}[b]{0.23\textwidth}
         \centering
         \includegraphics[width=\textwidth]{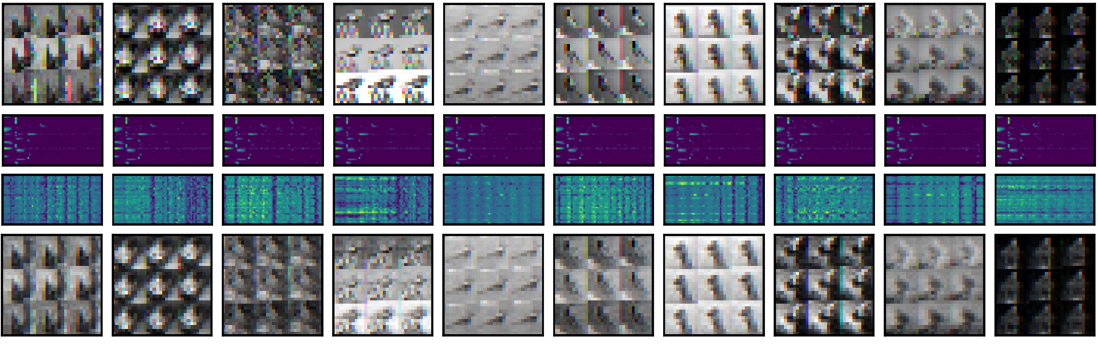}
     \end{subfigure}    
     \caption{The figure shows the data extraction and data reconstruction in CIFAR dataset. The first row shows the original data, the second row displays the invariant representations, the third row displays the variant representations, and the last row shows the reconstructed data.}
     \label{fig:UKIE-CIFAR}
\end{figure*}

{\color{duong}\subsubsection{Environment settings}
The transmit SNR is set to $20$dB for normal conditions and $5$dB for noisy conditions. Moreover, the receiver is assumed to have perfect channel parameter estimation.}

\subsection{Communication Efficiency}
\subsubsection{Compression Efficiency}\label{sec:efficient-compression}
{\color{duong}Figs.~\ref{fig:AWGN-20dB} and \ref{fig:AWGN-5dB} presents a comparison of UKIE with other semantic communication (SemCom) baselines over an AWGN channel with an SNR of 20 and 5 dB, where the compression ratio ranges from 0.04 to 0.4. For the digital communication system, a combination of 3/4 rate LDPC coding and 64-QAM modulation is employed. 
When the compression ratio is 0.4, all approaches trasnmit sufficient semantic information to effectively support the downstream task, leading to similarly high accuracy levels.
The results clearly indicate that the proposed UKIE method significantly outperforms other SemCom baselines, such as DeepJSCC and DeepSC. Additionally, UKIE demonstrates performance comparable to SemCC in terms of both classification accuracy and PSNR. The superior accuracy of SemCC is attributed to the robustness of contrastive learning, which has recently shown effectiveness in AI applications. 

When evaluating performance under more challenging Rayleigh fading channels (see Figs.~\ref{fig:Rayleigh-20dB} and \ref{fig:Rayleigh-5dB}), the robustness of UKIE becomes more evident. Specifically, UKIE consistently outperforms the compared methods in both accuracy and data reconstruction. This superior performance is attributed to UKIE's ability to extract and retain the most salient semantic information in its semantic memory. This stored knowledge remains resilient to the effects of the noisy physical channel, ensuring that the semantic information received by the decoder remains intact and lossless.
}
\begin{figure}[!h]
\centering
\includegraphics[width = 0.49\linewidth]{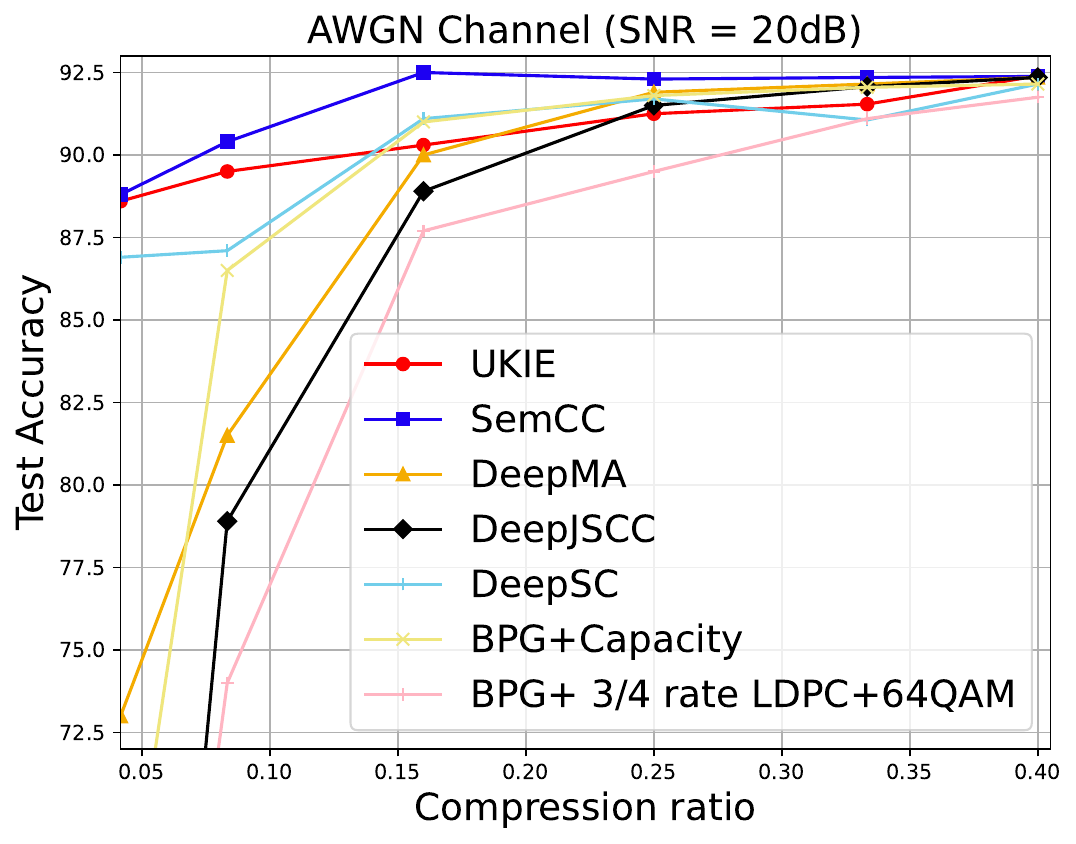}
\includegraphics[width = 0.49\linewidth]{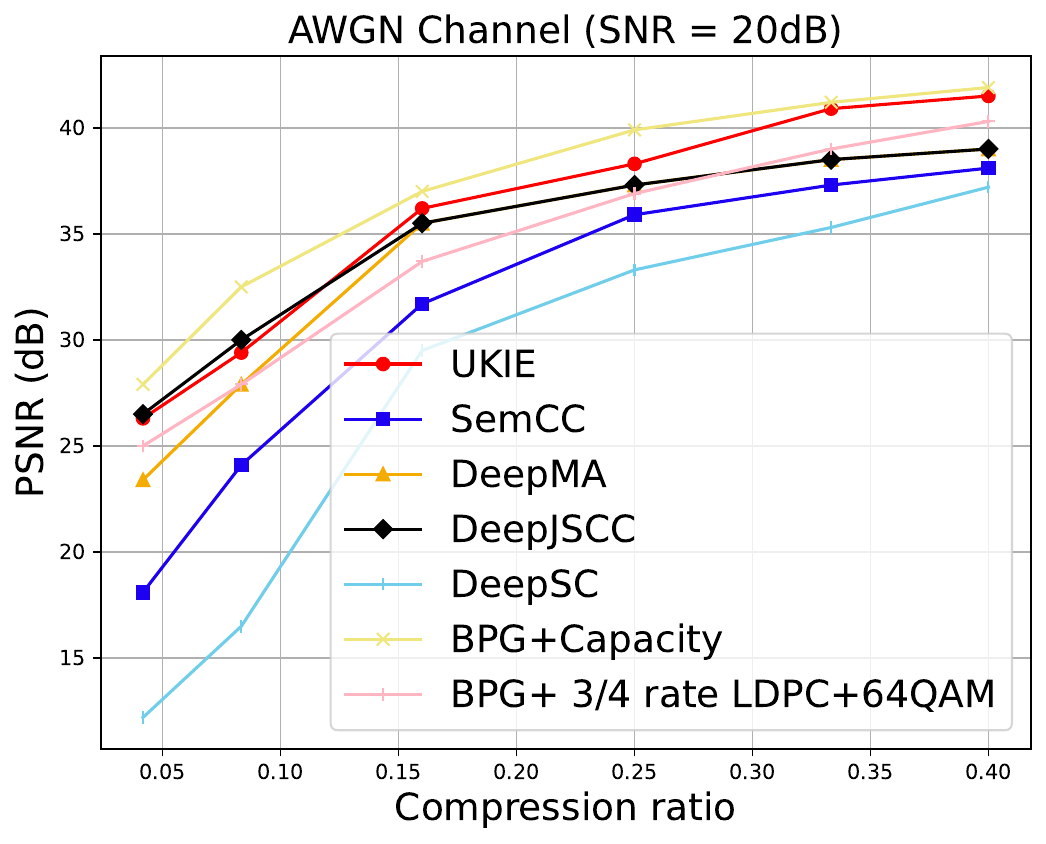}
\caption{Evaluation of data transmission under AWGN channel (SNR $= 20$dB) on CIFAR-10, the report in test accuracy (left figure) and PSNR (right figure).}
\label{fig:AWGN-20dB}
\end{figure}
\begin{figure}[!h]
\centering
\includegraphics[width = 0.49\linewidth]{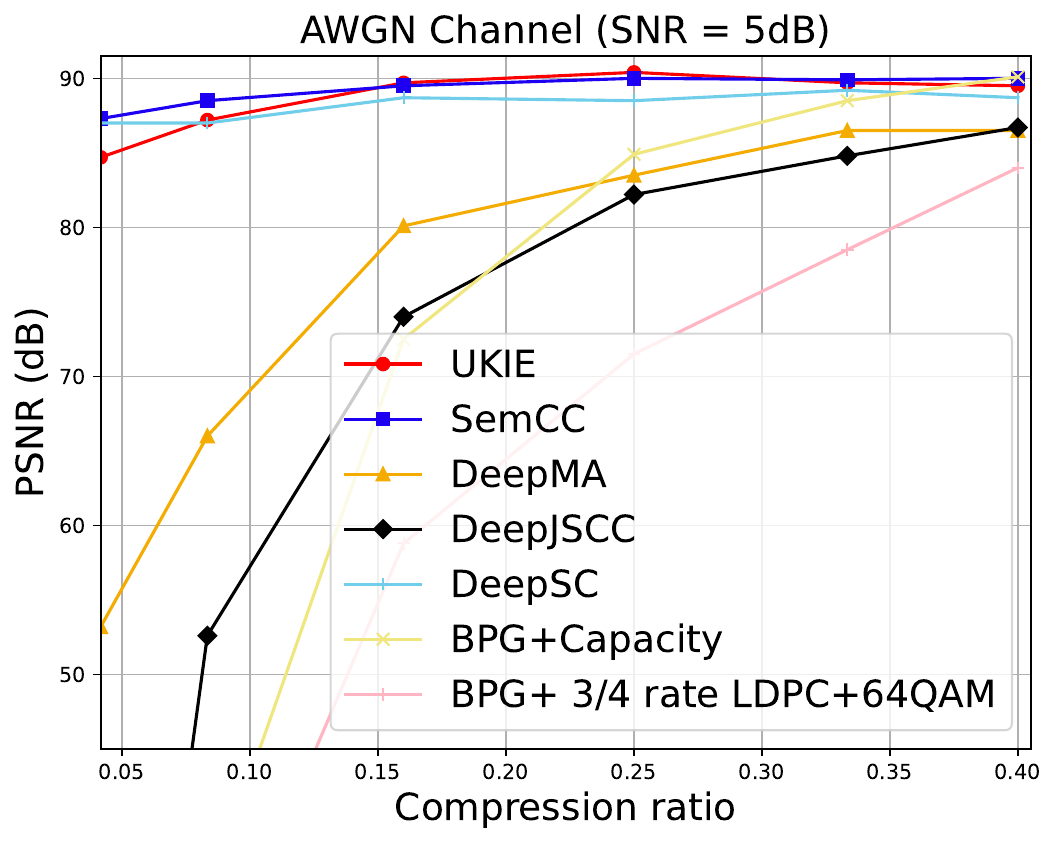}
\includegraphics[width = 0.49\linewidth]{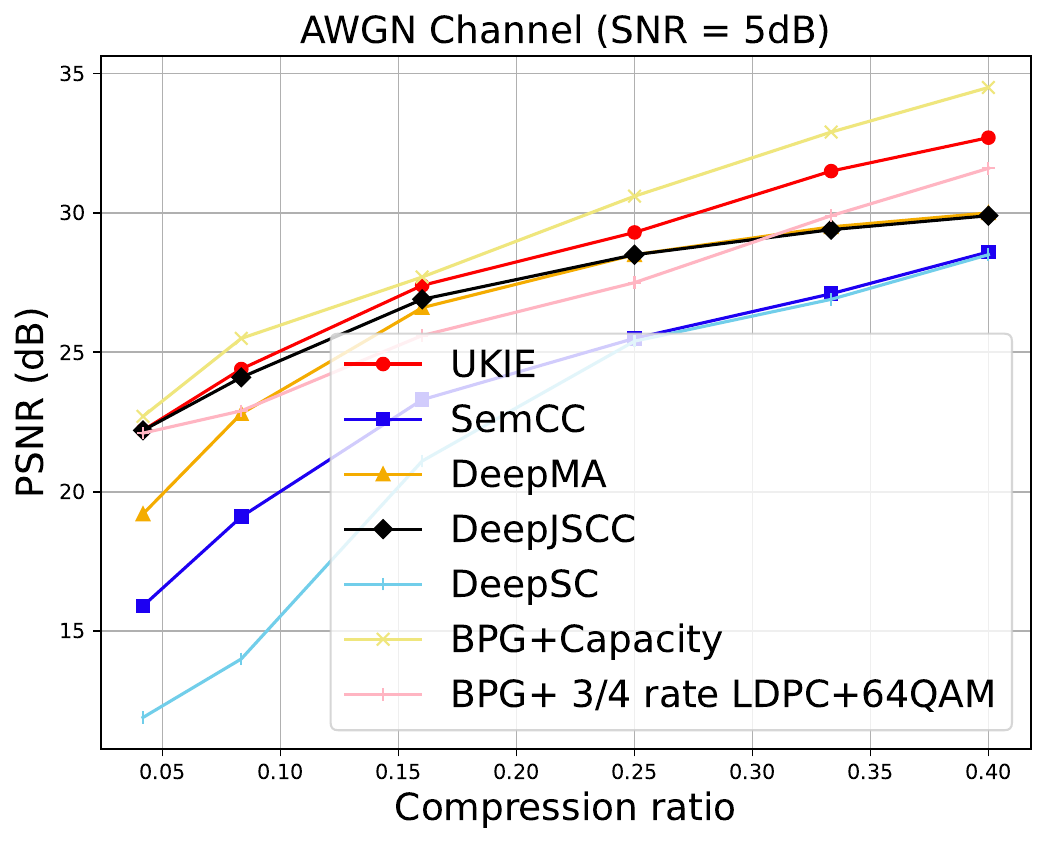}
\caption{Evaluation of data transmission under AWGN channel (SNR $= 5$dB) on CIFAR-10, the report in test accuracy (left figure) and PSNR (right figure).}
\label{fig:AWGN-5dB}
\end{figure}
\begin{figure}[!h]
\centering
\includegraphics[width = 0.49\linewidth]{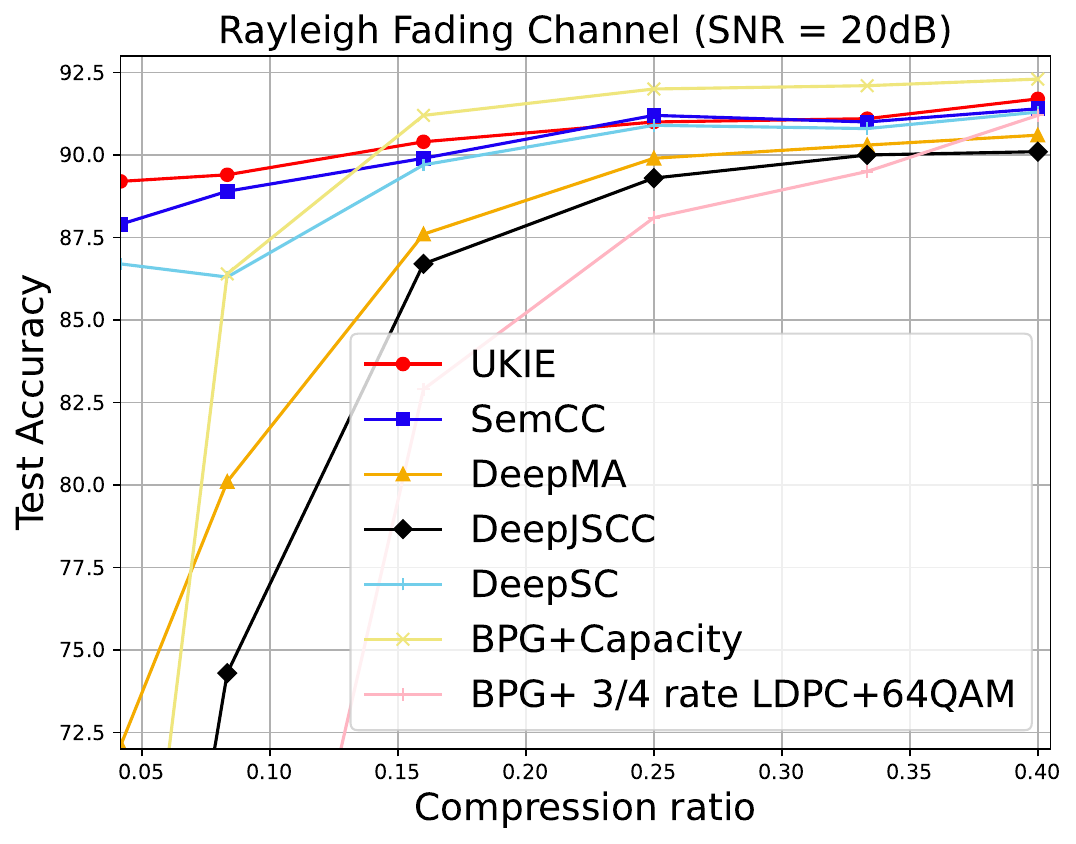}
\includegraphics[width = 0.49\linewidth]{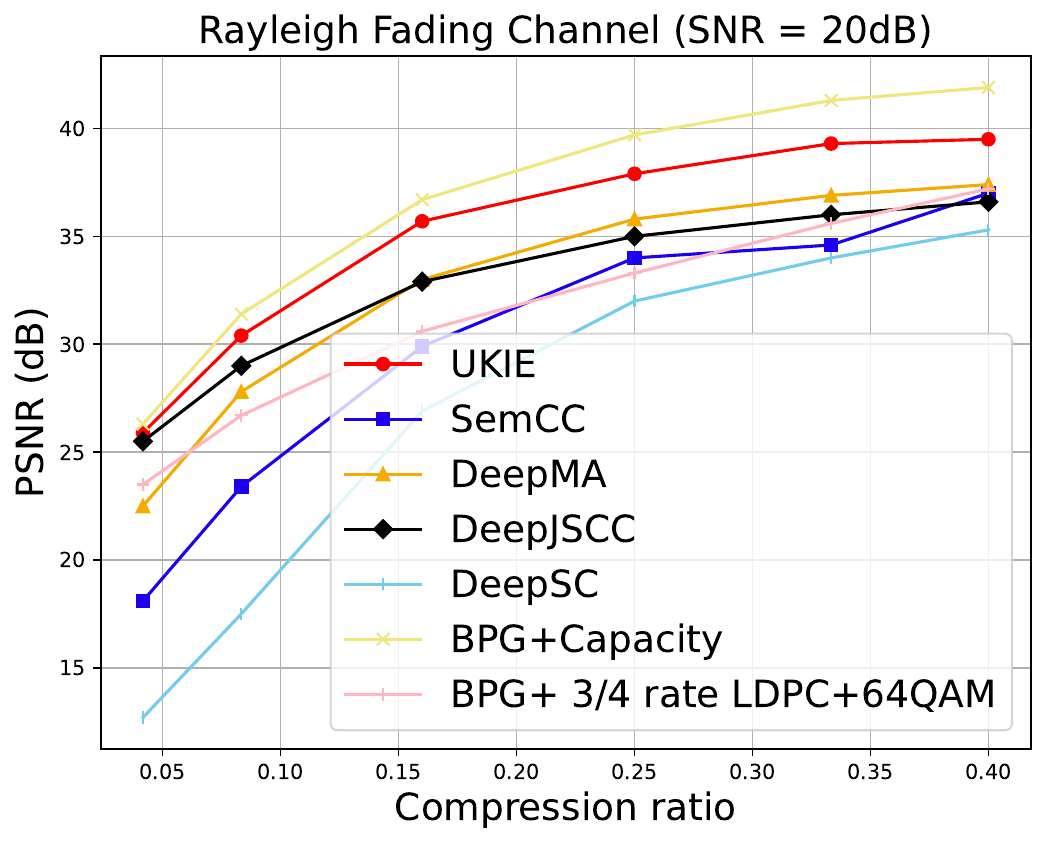}
\caption{Evaluation of data transmission under Rayleigh channel (SNR $= 20$dB) on CIFAR-10, the report in test accuracy (left figure) and PSNR (right figure).}
\label{fig:Rayleigh-20dB}
\end{figure}
\begin{figure}[!h]
\centering
\includegraphics[width = 0.489\linewidth]{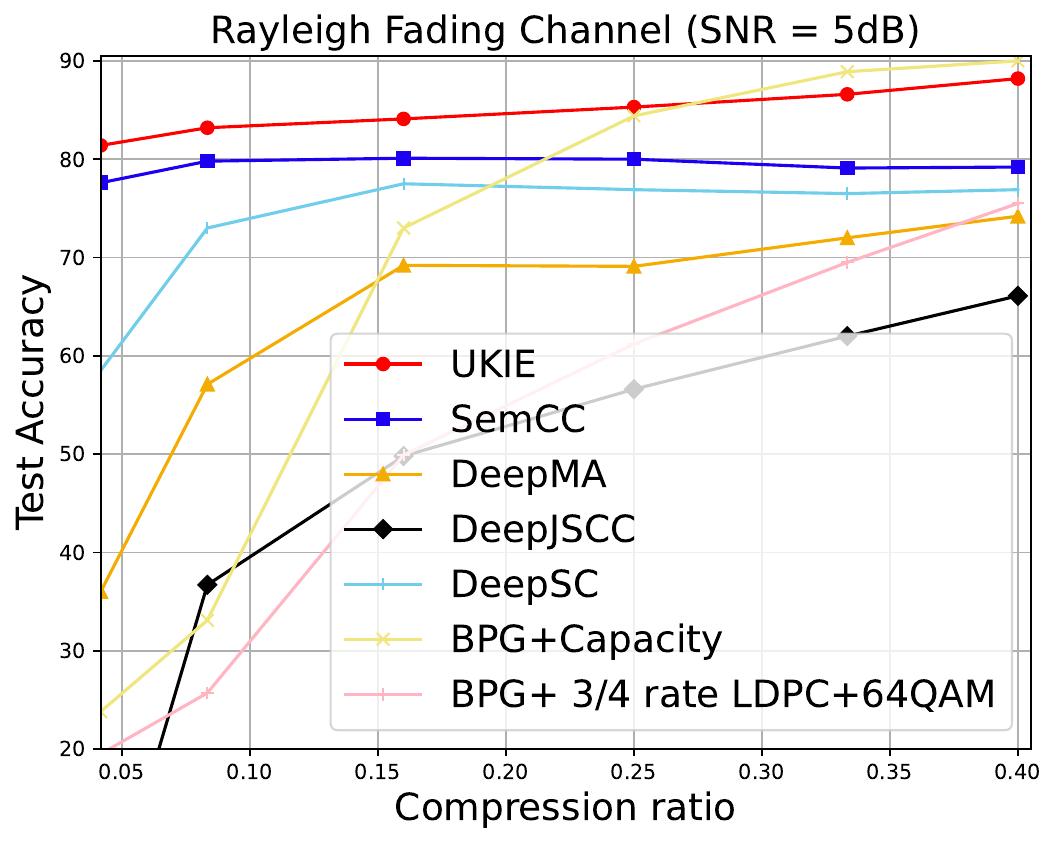}
\includegraphics[width = 0.496\linewidth]{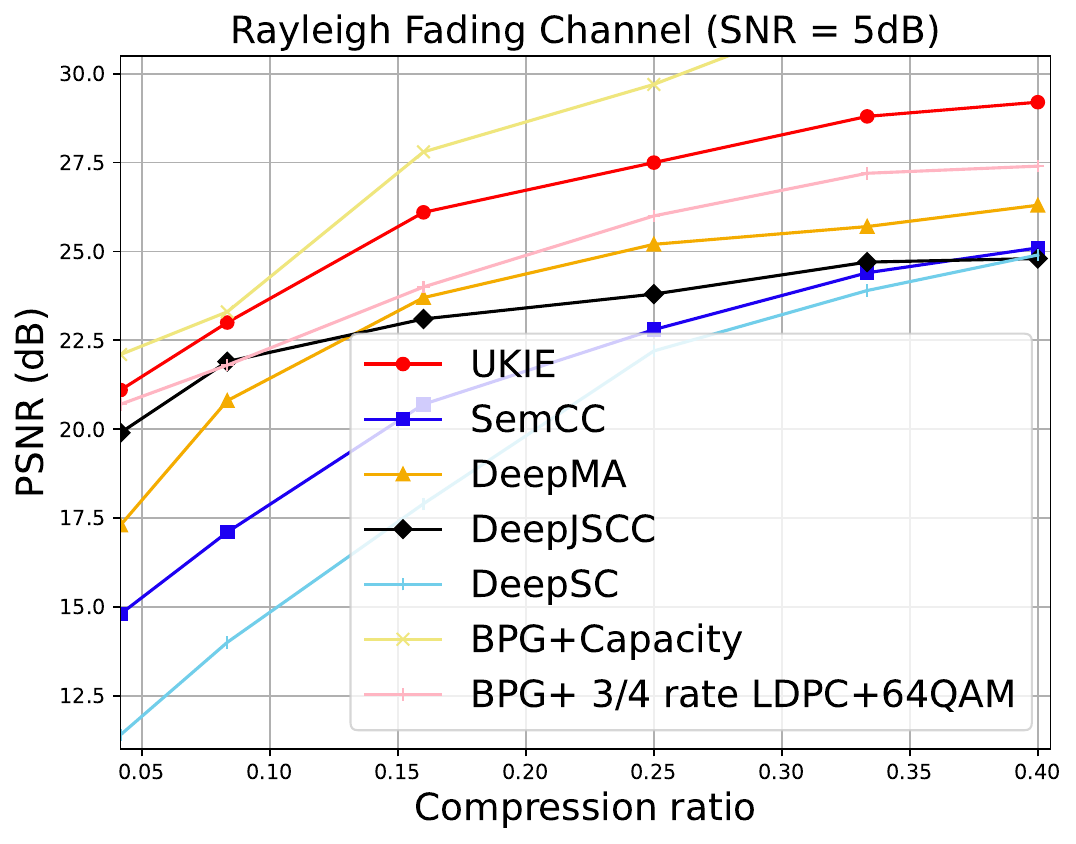}
\caption{Evaluation of data transmission under Rayleigh channel (SNR $= 5$dB) on CIFAR-10, the report in test accuracy (left figure) and PSNR (right figure).}
\label{fig:Rayleigh-5dB}
\end{figure}
\subsubsection{Communication Efficiency}
\label{sec:comm-eff}
{\color{duong} 
To evaluate the communication efficiency and generalization capability of UKIE, we trained the model under a fixed setting and tested it across different scenarios. Specifically, the semantic encoder and decoder were trained with a fixed SNR of 5 dB, and the compression ratios were set to 0.18 and 0.09, respectively. Figure~\ref{fig:Rayleigh-Compression} presents the test accuracy comparison under Rayleigh fading channels for various SNR levels.

The results indicate that the proposed approach achieves competitive test accuracy, as the semantic information is effectively extracted and stored in the semantic memory. The transmitted data primarily contains non-causal representations, which do not capture the most salient features for classification tasks. However, these non-causal representations remain important for data transmission and classification performance. This is evidenced by the decline in test accuracy when using a high compression ratio, particularly when the SNR approaches 0 dB.
These findings suggest that proposed UKIE maintains robustness in real-world conditions.
}
\begin{figure}[!h]
\centering
\includegraphics[width = 0.49\linewidth]{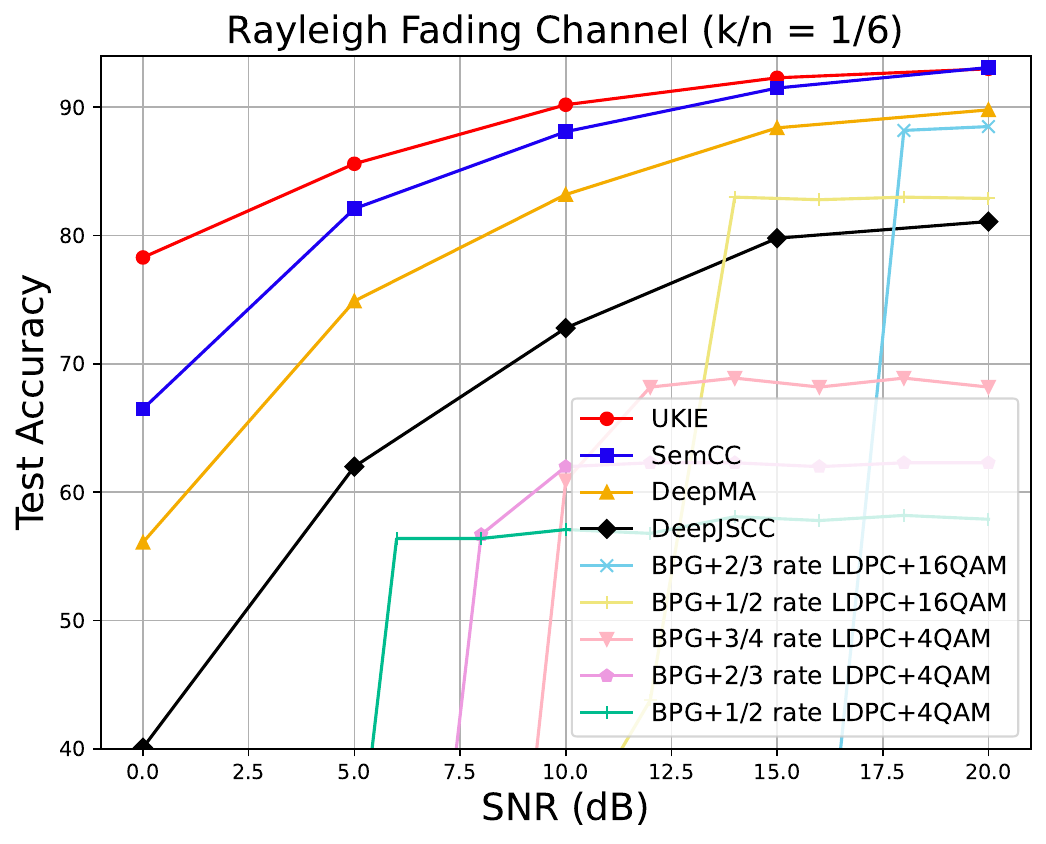}
\includegraphics[width = 0.49\linewidth]{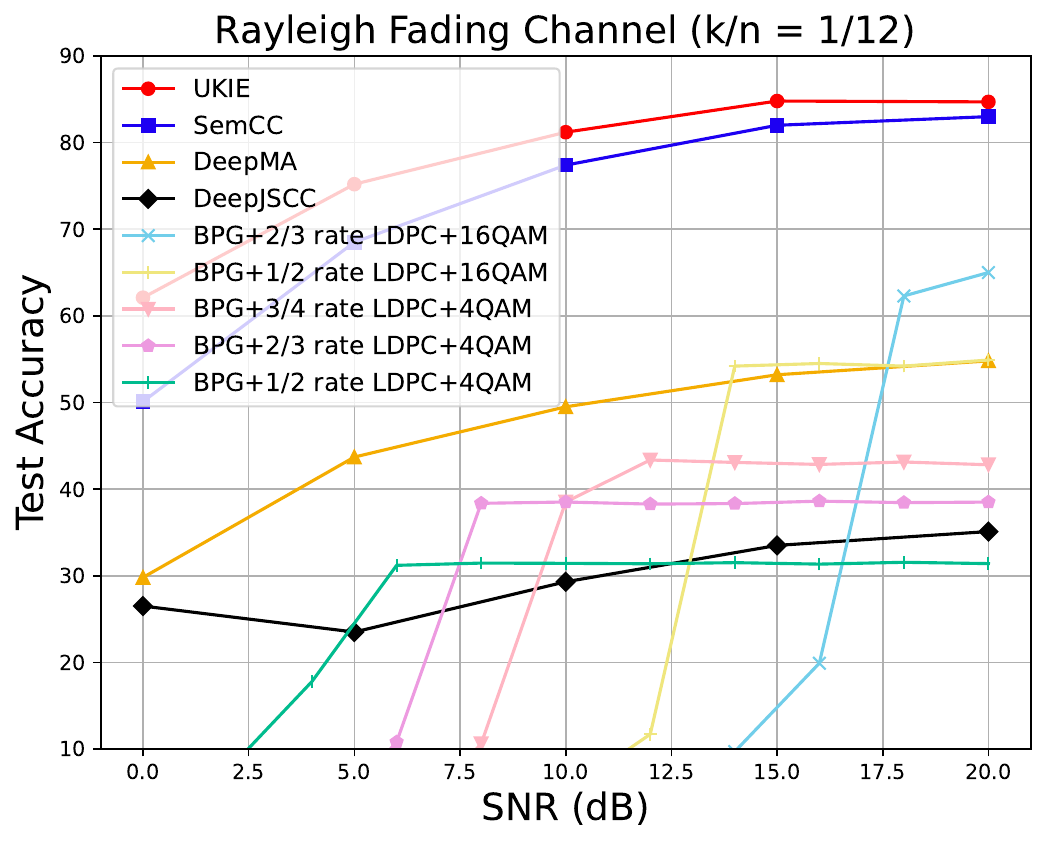}
\caption{Evaluation of data transmission under Rayleigh channel on CIFAR-10, the report in test accuracy with compression ratio = $0.18$ (left figure) and $0.09$ (right figure). The semantic encoder and decoder are trained at SNR = 5dB.
}
\label{fig:Rayleigh-Compression}
\end{figure}

\subsection{Evaluations on Computational Efficiency}\label{sec:efficient-computation}
{\color{duong}\begin{table}[h]
    \centering
    \caption{Comparison of the Computational Efficiency Characteristics of UKIE with Baseline Methods}
    \label{tab:comp-eff}
    \resizebox{1\columnwidth}{!}{%
    \begin{tabular}{l l l l l l l}
        \toprule
        \textbf{Method} & \textbf{Enc.} & \textbf{Dec.} & {\textbf{\# of Params.}} & \textbf{Train} & \textbf{Test}   \\
                        & \textbf{Model} & \textbf{Model} & \textbf{Enc. / Dec.} & \textbf{times} & \textbf{times}\\
        \midrule
        UKIE        & ResNet9$\times2$ & ResNet18 & 10.2M/9.8M                  & 1h08m & 2.23ms \\
        DeepSC      & ViT-12           & ViT-12             & 85.1M/85.1M      & 4h32m & 5.27ms \\
        DeepJSCC    & ResNet18         & ResNet18   & 11.4M/11.4M      & 0h59m & 2.71ms \\
        DeepMA      & ResNet18         & ResNet18   & 11.4M/11.4M      & 1h02m & 2.63ms \\
        SemCC       & ResNet18         & ResNet18   & 11.4M/11.4M      & 1h39m & 2.87ms \\
        RS coding          & N/A       & N/A     & N/A   & N/A & 4.14ms\\
        Turbo coding       & N/A       & N/A     & N/A   & N/A & 8.59ms\\
        \bottomrule
    \end{tabular}}
\end{table}
Table~\ref{tab:comp-eff} summarizes of the computational costs of UKIE compared to various baseline methods, including the number of parameters, as well as training and testing times. UKIE requires a comparable number of parameters to DeepJSCC, DeepMA, and SemCC. This is primarily because UKIE employs two shallow ResNet9 networks for data and knowledge extraction, in contrast to the deeper ResNet18 used in other methods.
Additionally, by leveraging two shallow ResNet9 networks, UKIE achieves slightly faster testing times than the other baselines. In our study, we adopt the framework of DeepSC \cite{2021-SEM-DeepSC}, where a transformer architecture is used to extract semantic information. However, we utilize a Vision Transformer (ViT) \cite{2021-ViT-VisionTransformer} to fit our evaluations in image dataset instead of a standard transformer, which is typically used for text datasets. Due to the higher computational complexity of ViTs, DeepSC requires significantly more computation resources to achieve optimal performance.}




\subsection{Performance of UKIE on different challenging AI dataset}
\subsubsection{Robustness of Causality Invariance Learning}
\label{sec:retreiving-symmetries}
Table~\ref{tab:dataset-learning} illustrates the learning performance of UKIE across various datasets. UKIE successfully achieves robust invariance learning, as evidenced by the notably low invariant loss ($\mathcal{L}_{\textrm{iv}}\sim 0.03$). At the same time, the variant loss exhibits a contrasting high value ($\mathcal{L}_{\textrm{v}}\sim 0.98$), indicating a clear separation between invariant and variant representations. This distinction contributes significantly to the meaningful information related to the input data (mentioned in Lemma~\ref{lemma:meaningful-representations}). The efficient data reconstruction, indicated by the low reconstruction loss ($\mathcal{L}_{\textrm{rec}} \sim 0.03$), further supports these meaningful information properties.
\begin{table}[!ht]
\centering
\caption{Experiments on the robustness of causality invariance learning. $\downarrow$: lower is better. $\uparrow$: higher is better. 95\% confidence intervals are from 5 trials.}
\label{tab:dataset-learning}
\resizebox{1\columnwidth}{!}{%
  \begin{tabular}{lccccc}
    \toprule
    Method & $\mathcal{L}_{\textrm{rec}}~\downarrow$ & $\mathcal{L}_{\textrm{iv}}~\downarrow$ & $\mathcal{L}_{\textrm{v}}~\uparrow$ & Accuracy $\uparrow$ \\
    \midrule
    EMNIST   & $0.00049 \pm  0.001$ & $1.5\times 10^{-5}$ & $0.99$ & $99.18\% \pm 0.21$\\
    CIFAR-10 & $0.00072 \pm 0.0008$ & $1.7\times 10^{-3}$ & $0.99$ & $73.41\% \pm 1.67$\\
    CINIC-10 & $0.00032 \pm 0.00012$ & $1.5\times 10^{-5}$ & $0.99$ & $72.67\% \pm 2.13$\\
    CELEB-A  & $0.0007 \pm 0.001$ & $2.0\times 10^{-5}$ & $0.98$ & $72.67\% \pm 1.20$\\
    \bottomrule
  \end{tabular}}
\end{table}
\begin{table}[!ht]
\centering
\caption{Comparison to prior VAE-based and GAN-based representation learning methods. The evaluations are measured in terms of PSNR. $95\%$ confidence intervals are from 5 trials.}
\label{tab:comparison-vae}
\resizebox{1\columnwidth}{!}{%
\begin{tabular}{lcccc}
    \toprule
    Method & MNIST & EMNIST & CIFAR-10 & CINIC-10 \\
    \midrule
    AE  & $35.32 \pm 0.12$ & $34.71 \pm 0.42$ & $33.97 \pm 0.42$ & $32.63 \pm 0.14$ \\
    VAE \cite{2013-DL-VAE} & $32.26 \pm 0.12$ & $33.71 \pm 0.42$ & $33.05 \pm 0.42$& $34.98 \pm 0.14$ \\
    Factor-VAE \cite{2018-DL-DisentanglingFactorising} & $34.32 \pm 0.07$ & $32.56 \pm 0.05$ & $34.82 \pm 0.05$ & $34.74 \pm 0.09$ \\
    Dip-VAE \cite{dipave} & $34.38 \pm 0.08$ & $35.26 \pm 0.08$ & $33.02 \pm 0.08$ & $33.97 \pm 0.06$ \\
    $\beta$-VAE \cite{2016-DL-BetaVAE} & $35.76 \pm 0.12$ & $35.02 \pm 0.03$ & $34.36 \pm 0.03$ & $34.25 \pm 0.06$ \\
    InfoGan \cite{2016-GAN-InfoGan} & $35.26 \pm 0.02$ & $36.72 \pm 0.07$ & $36.06 \pm 0.12$ & $37.15 \pm 0.08$ \\
    Unsupervised GAN\cite{2020-GAN-UnGan} & $37.86 \pm 0.01$ & $37.62 \pm 0.03$ & $37.56 \pm 0.11$ & $38.28 \pm 0.07$\\
    Intepretable GAN\cite{2021-GAN-IPGAN} & $38.72 \pm 0.05$ & $38.91 \pm 0.06$ & $38.02 \pm 0.07$ & $37.42 \pm 0.12$\\
    UKIE (Ours) & $\textbf{42.27} \pm 0.04$ & $\textbf{41.83} \pm 0.04$& $\textbf{39.29} \pm 0.03$ & $\textbf{38.41} \pm 0.03$ \\
    \bottomrule
\end{tabular}}
\end{table}
\subsubsection{{\color{duong}Why do we not use other data reconstruction baselines?}}
\label{sec:decode-with-knowledge}
Table~\ref{tab:comparison-vae} shows the data reconstruction efficiency of UKIE compared with those of other VAE-based architectures (i.e., VAE \cite{2013-DL-VAE}, Factor-VAE \cite{2018-DL-DisentanglingFactorising}, Lie-VAE \cite{2021-DL-LieDisentanglement}, $\beta$-VAE \cite{2016-DL-BetaVAE}). In our experiments, we consider the data reconstruction efficiency in terms of PSNR. UKIE model demonstrates a competitive level of PSNR compared to other methods. Notably, UKIE achieves significantly higher PSNR compared to VAE. This indicates that a substantial amount of important information is effectively transmitted to the pair of invariant-variant representations that are extracted appropriately compared to bottlenecked representations from other state-of-the-art baselines.

\subsection{Ablation Test}
\label{sec:ablation-test}

\subsubsection{Information Bottleneck in Semantic Encoder} 
We investigate the compact and structured properties of invariant and variant representations by evaluating UKIE with different embedding layer sizes. Our experiments focus on varying channel sizes, examining configurations with $C=4,8,16,24,32,48,64$ channels. By adjusting the channel size $C$, we calculate the total latent dimension $d_z = d_{z_V} + d_{z_K}$ as the sum of the dimensions of invariant and variant representations, given by $C\times 8\times 8 = 64\times C$. The results are shown in Fig.~\ref{fig:IB}.

From Fig.~\ref{fig:IB}, it is evident that the accuracy of invariant classification significantly decreases when the channel size of the latent representations is reduced to $C\leq 16$. This decrease in performance is attributed to the corresponding reduction in the size of the invariant representations, which leads to insufficient information for data inference. Conversely, the data reconstruction performance across different levels of information bottlenecks is stable. This stability is attributed to the variant representation, which provides adequate information for joint data reconstruction.

These findings suggest that our causality invariance learning approach can be enhanced by optimizing the balance between invariant knowledge and variant data sizes. By simultaneously increasing the invariant knowledge size and reducing the variant data size, sufficient data recognition can be achieved with lower communication costs over physical channels. This saved capacity can then be allocated to semantic knowledge, which is sparsely updated through the semantic channels. Consequently, this approach can significantly boost the communication efficiency of the SemCom system.
\begin{figure}[!ht]
\centering
\includegraphics[width = 0.47\linewidth]{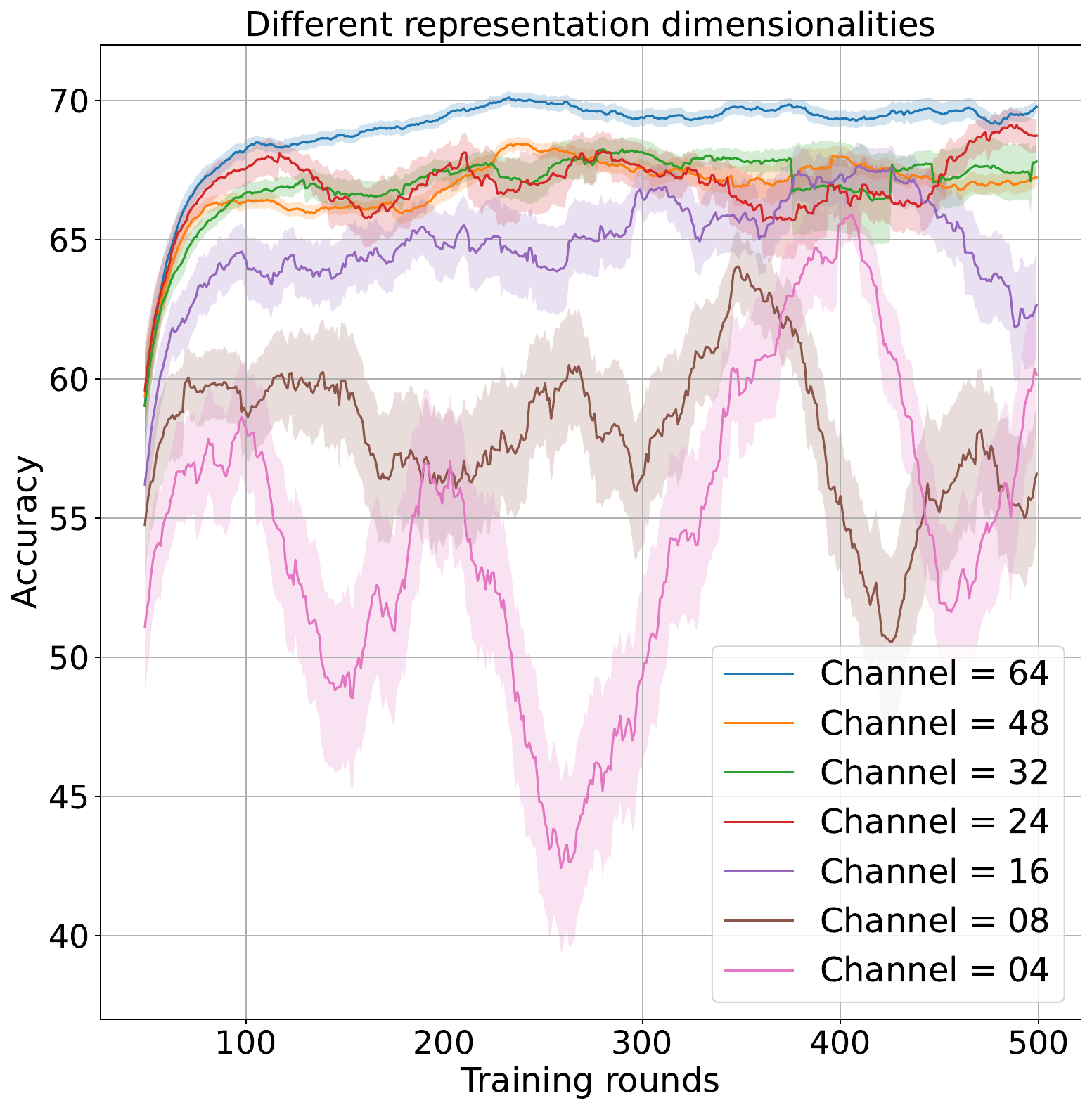}
\includegraphics[width = 0.5\linewidth]{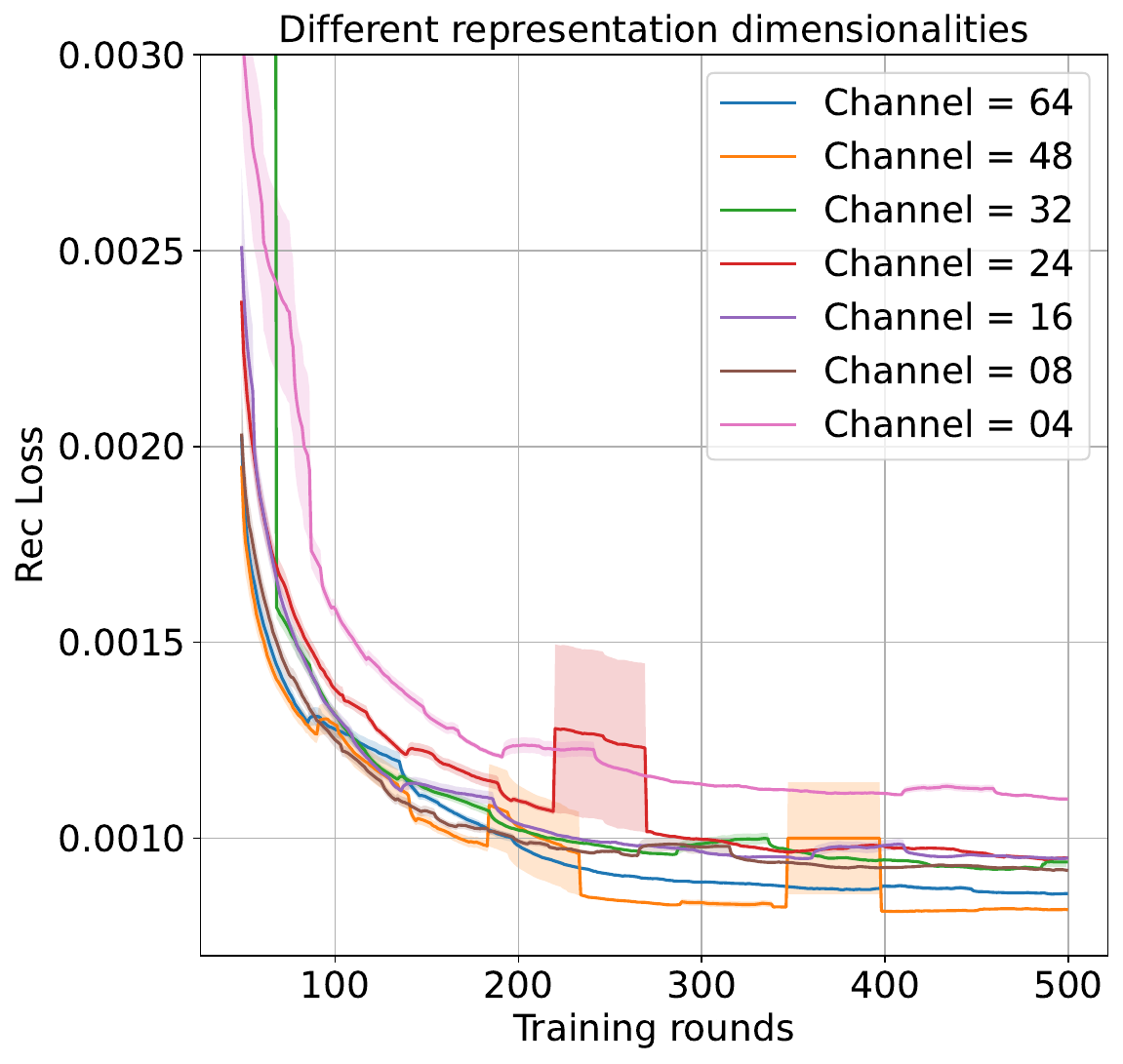}
\caption{Illustrations of the performance of Informational Bottlenecks.}
\label{fig:IB}
\end{figure}
\subsubsection{Assessing the Extent of Knowledge Extraction from Data}
\begin{table}[!ht]
\caption{Assessment of UKIE on different domains. The assessment is performed on Colored-MNIST, where each domain represents the images with a different degree of correlation between color and label. We assess two characteristics: 1) Accuracy: we use only invariant representations to predict the labels in each domain, and 2) Invariance: we measure the disparity between each generated $z^i_K$ within the target domain $i$ and the invariant representations $z^j_K$ generated from the dataset of the training domain $j$.}
\label{tab:ukie-domain}
\centering
\begin{tabular}{lccc}
\hline
\textbf{Algorithm} & $\textbf{+90\%}$ & $\textbf{+80\%}$ & $\textbf{-90\%}$ \\ \hline \hline
\multicolumn{4}{c}{\textbf{Domain 1 (+90\%})}                                                                                             \\ \hline
Accuracy                            & 89.5 $\pm$ 0.1                  & 82.3 $\pm$ 0.5                  & 73.2 $\pm$ 0.3                  \\
Invariant                           & 0.0003 $\pm$ 10e-5              & 0.0002 $\pm$ 10e-5              & 0.0002 $\pm$ 10e-5              \\ \hline \hline
\multicolumn{4}{c}{\textbf{Domain 2 (+80\%})}                                                                                  \\ \hline
Accuracy                            & 73.9 $\pm$ 0.1                  & 82.3 $\pm$ 0.5                  & 73.2 $\pm$ 0.3                  \\
Invariant                           & 0.0003 $\pm$ 10e-5              & 0.0002 $\pm$ 10e-5              & 0.0002 $\pm$ 10e-5              \\ \hline \hline
\multicolumn{4}{c}{\textbf{Domain 3 (-90\%})}                                                                                  \\ \hline
Accuracy                            & 89.5 $\pm$ 0.1                  & 80.7 $\pm$ 0.5                  & 73.2 $\pm$ 0.3                  \\
Invariant                           & 0.0003 $\pm$ 10e-5              & 0.0002 $\pm$ 10e-5              & 0.0002 $\pm$ 10e-5              \\ \hline
\end{tabular}
\end{table}
To determine the extent of invariant knowledge extraction from the data and the appropriate semantic knowledge size to represent the data from each label, we keep the size of the embedding layer constant (i.e., $C_\textrm{IB}=32$), while simultaneously varying the number of invariant channels and setting the variant channel size to $C_\textrm{v} = C_\textrm{IB} - C_\textrm{iv}$. This approach is adopted to ensure that the performance of data reconstruction remains unaffected by the informational bottleneck. 
The results are shown in Fig.~\ref{fig:unbiased-knowledge}. 
As can be seen, UKIE performs optimally when the number of invariant knowledge channels are set to $C \geq 16$. Additionally, the data reconstruction remains consistent as the invariant knowledge size is increased. This indicates that higher data reconstruction efficiency can be achieved by using a larger invariant knowledge size (e.g., $C_\textrm{iv} \geq 24$, $C_\var \leq 8$). This also implies that a high compression ratio of $(3\times 32\times 32) : (8\times 64) = 6:1$ can be attained without quality loss in the variant data, which needs to be transmitted over the physical channel.
\begin{figure}[!ht]
\centering
\includegraphics[width = 0.475\linewidth]{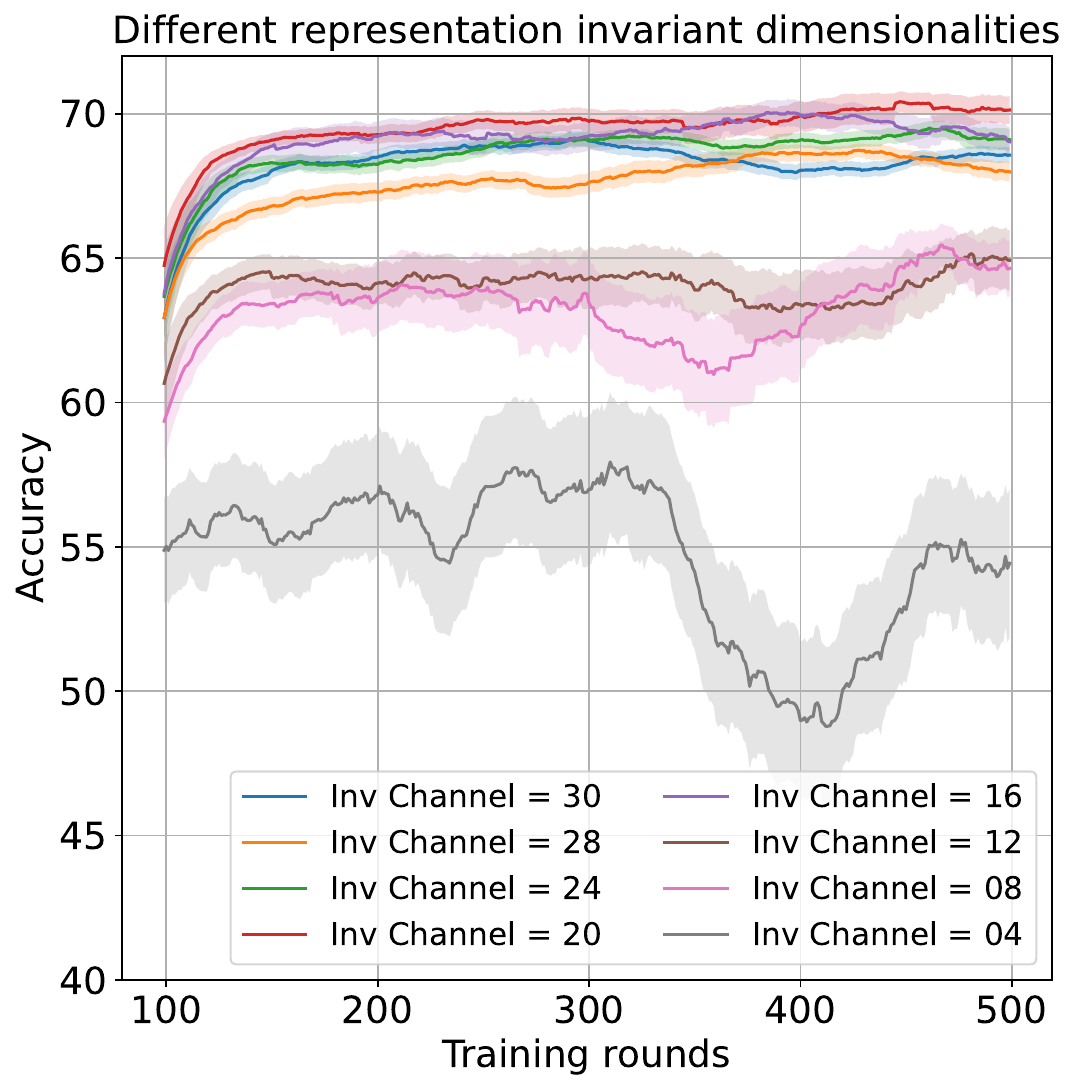}
\includegraphics[width = 0.51\linewidth]{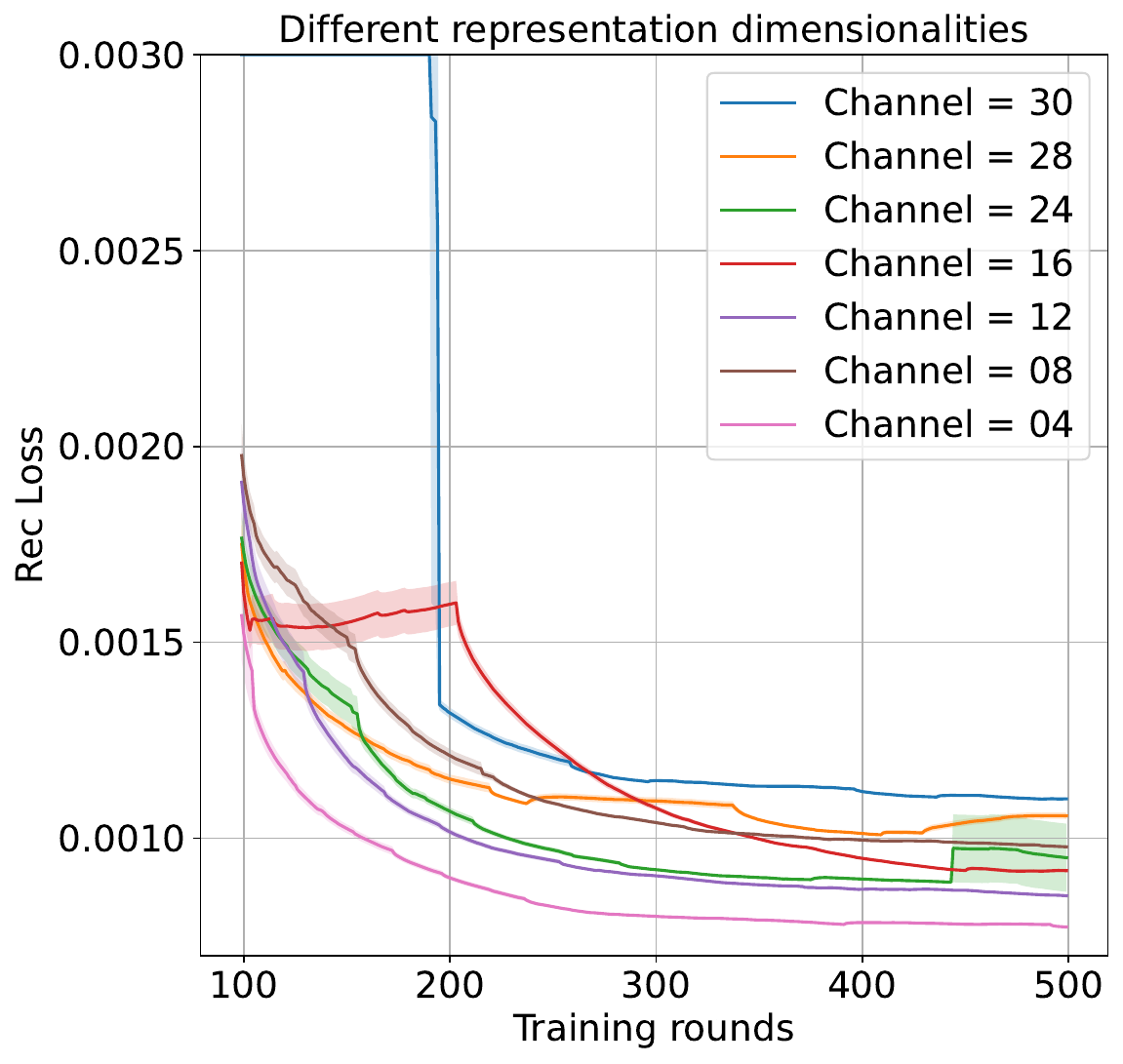}
\caption{Illustrations of the efficacy of invariant knowledge encoder.}
\label{fig:unbiased-knowledge}
\end{figure}


\subsubsection{Ablation test on GAN via adjusting loss coefficients}
We assess the training of the UKIE model under various settings of the $\alpha_\textrm{gtc}$, $\alpha_\textrm{v}$, and $\alpha_\textrm{iv}$ parameters. The objective is to optimize the model by understanding the impact of each component of training loss on UKIE's overall performance. \emph{Our aim is to reduce negative transfer between tasks.} We start with the initial setting of $\alpha_\textrm{gtc}=1$, $\alpha_\textrm{rec}=1$, $\alpha_\textrm{iv}=1$, $\alpha_\textrm{v}=1$. Fig.~\ref{fig:learning-coeff} reveals the results of training under different settings of the four coefficients. 
\begin{figure}[!ht]
\centering
\includegraphics[width = 1\linewidth]{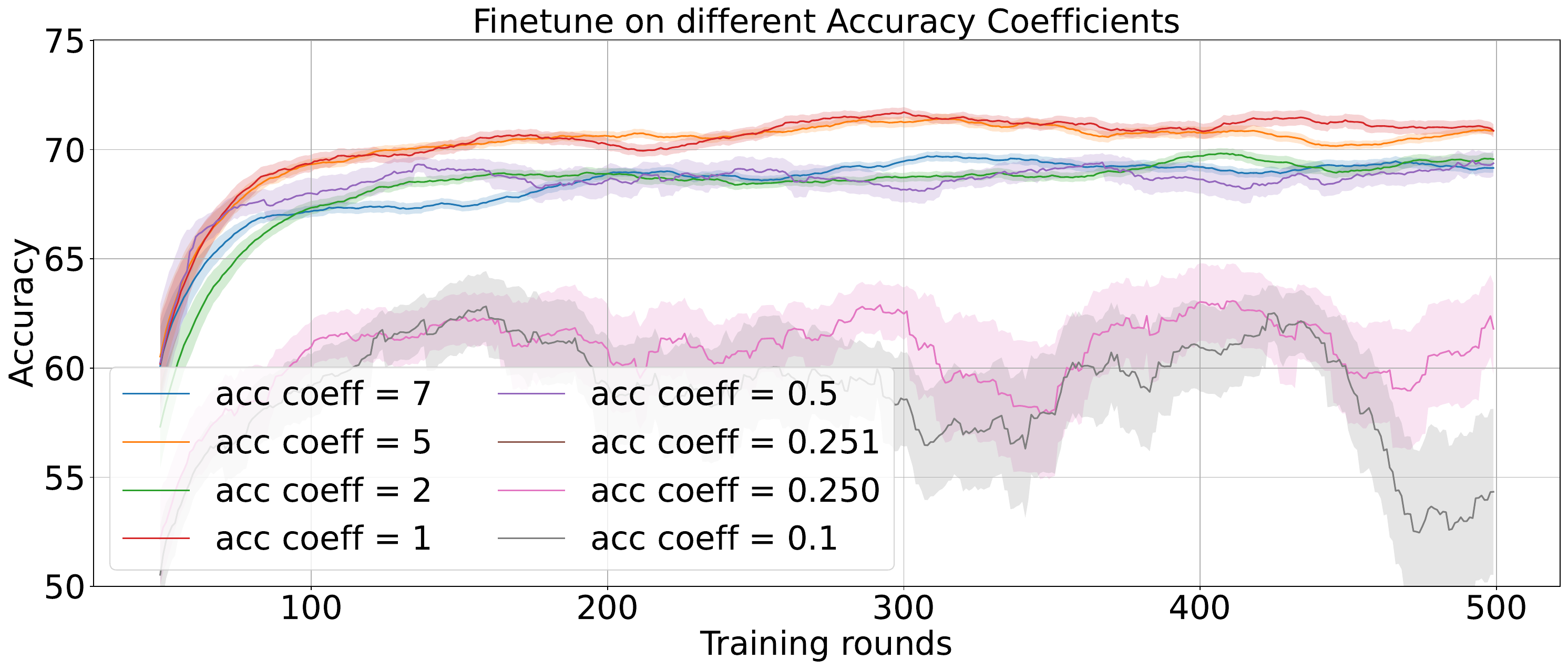} \\
\includegraphics[width = 0.49\linewidth]{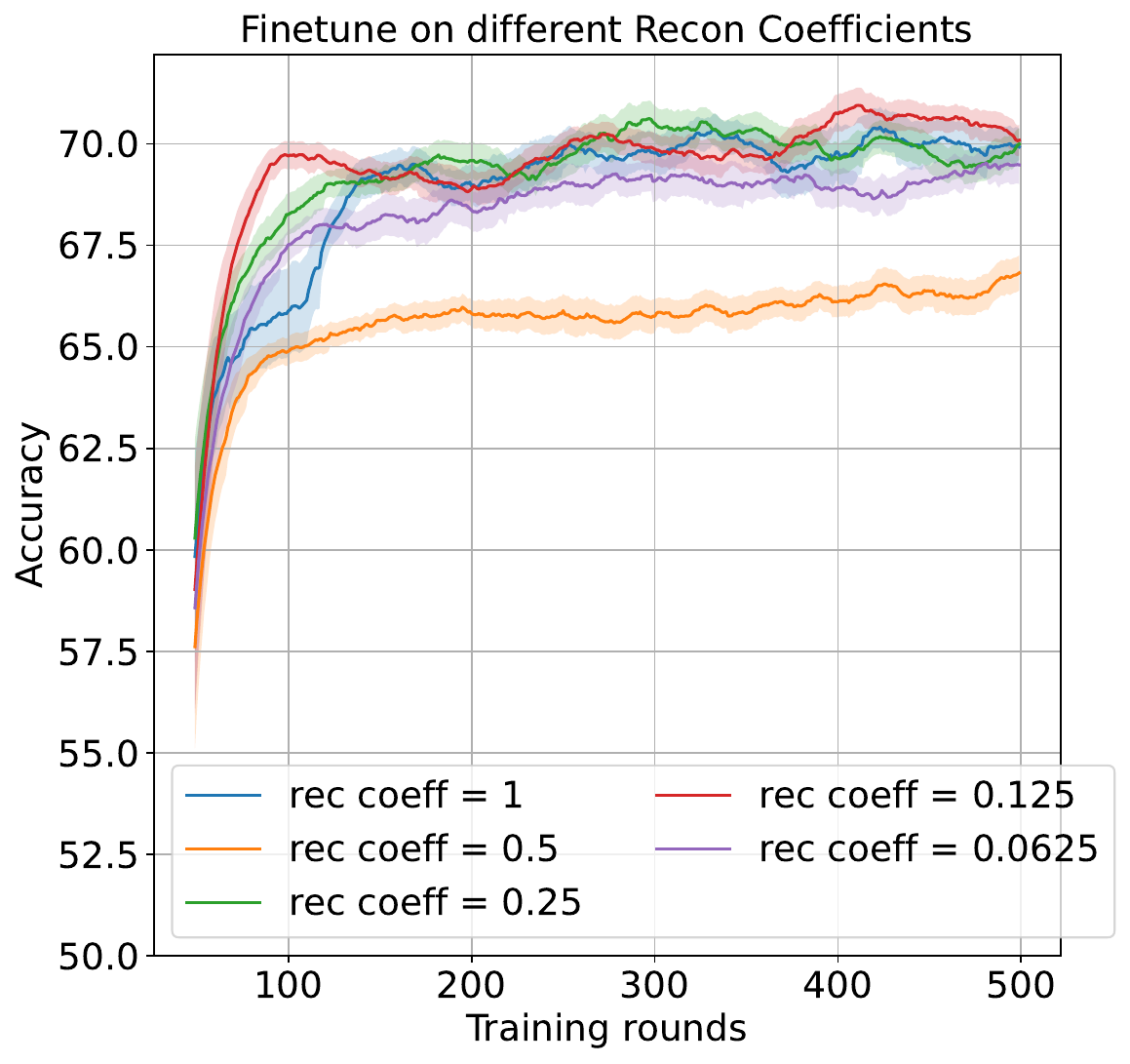} 
\includegraphics[width = 0.47\linewidth]{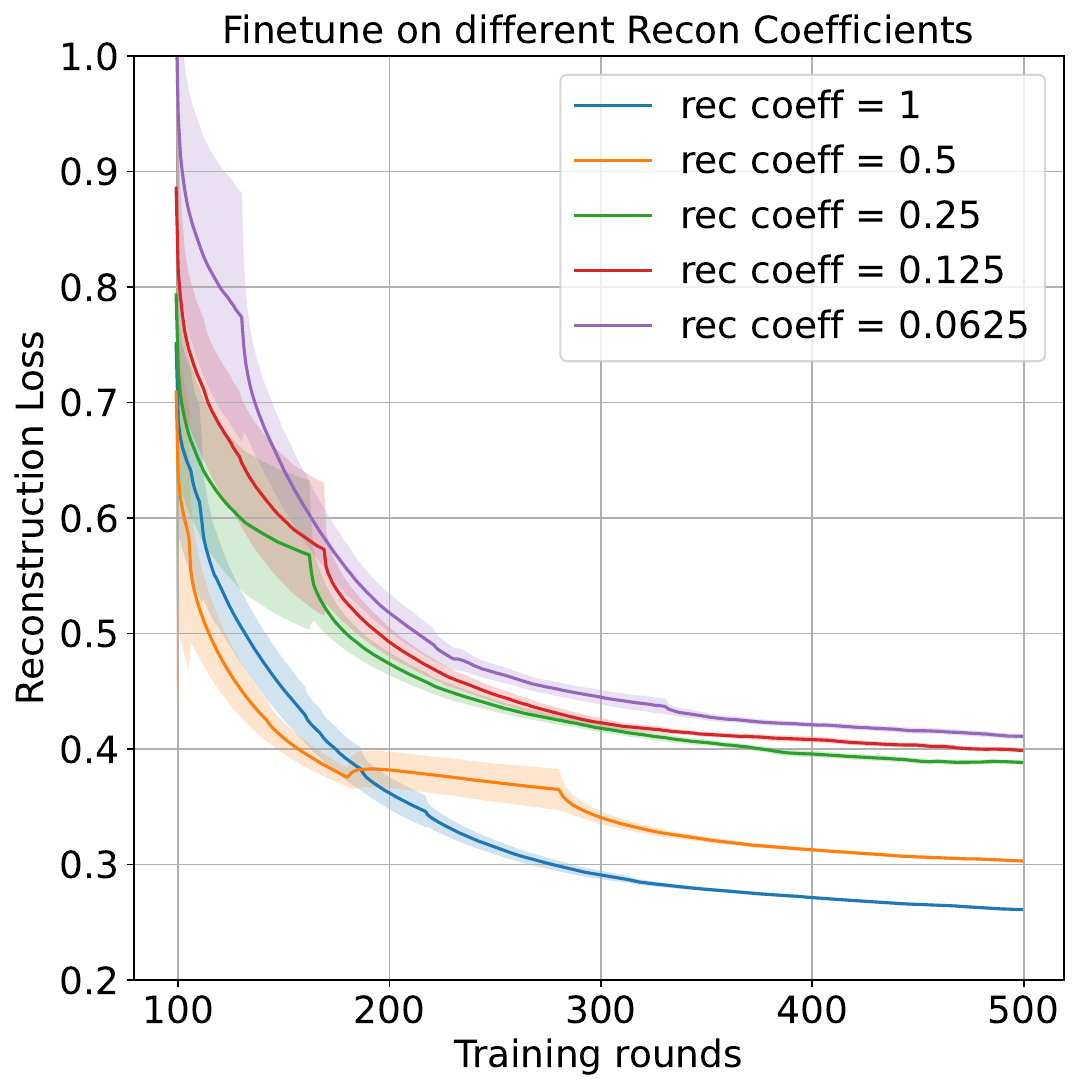} \\
\includegraphics[width = 0.47\linewidth]{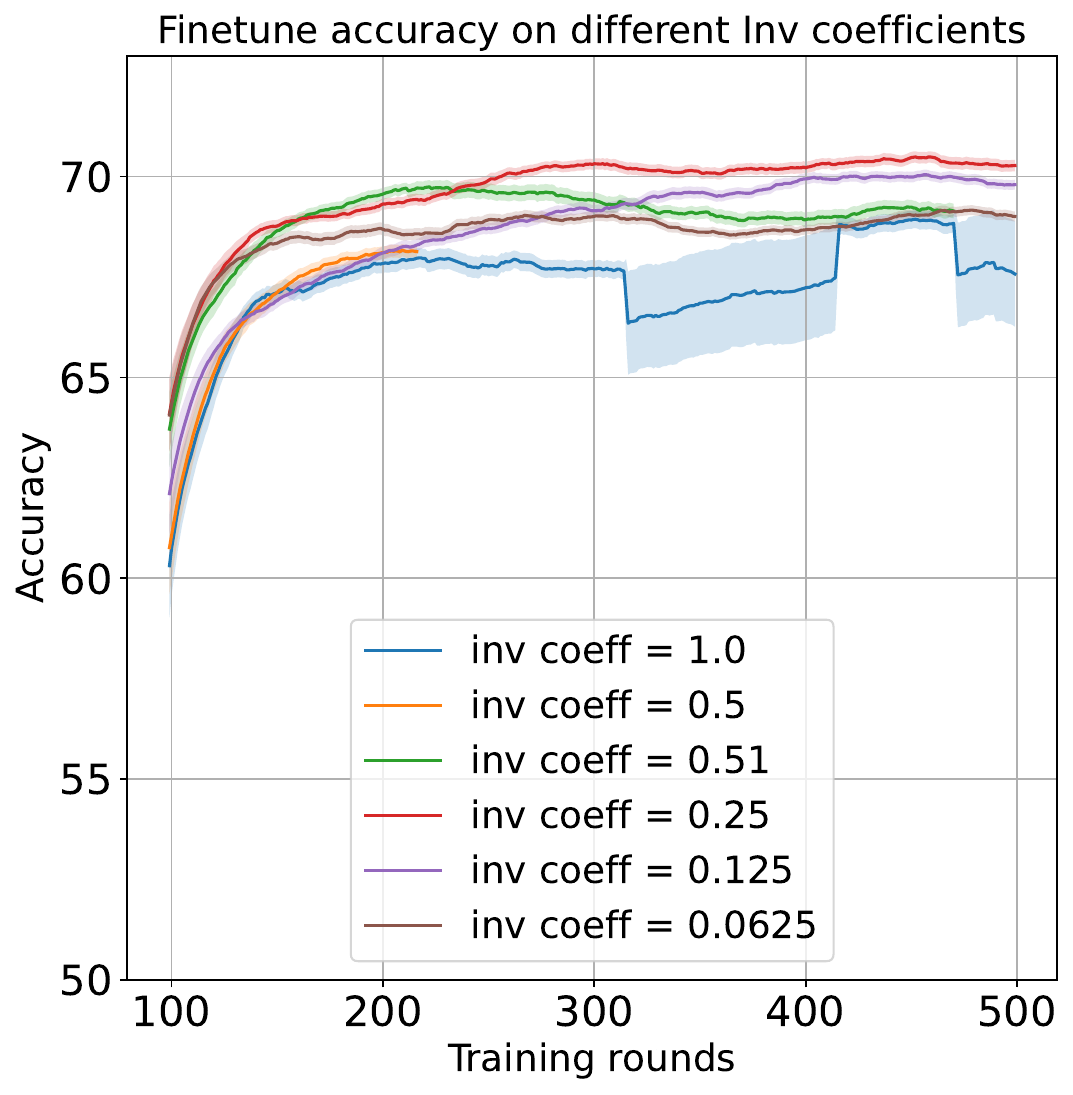} 
\includegraphics[width = 0.5\linewidth]{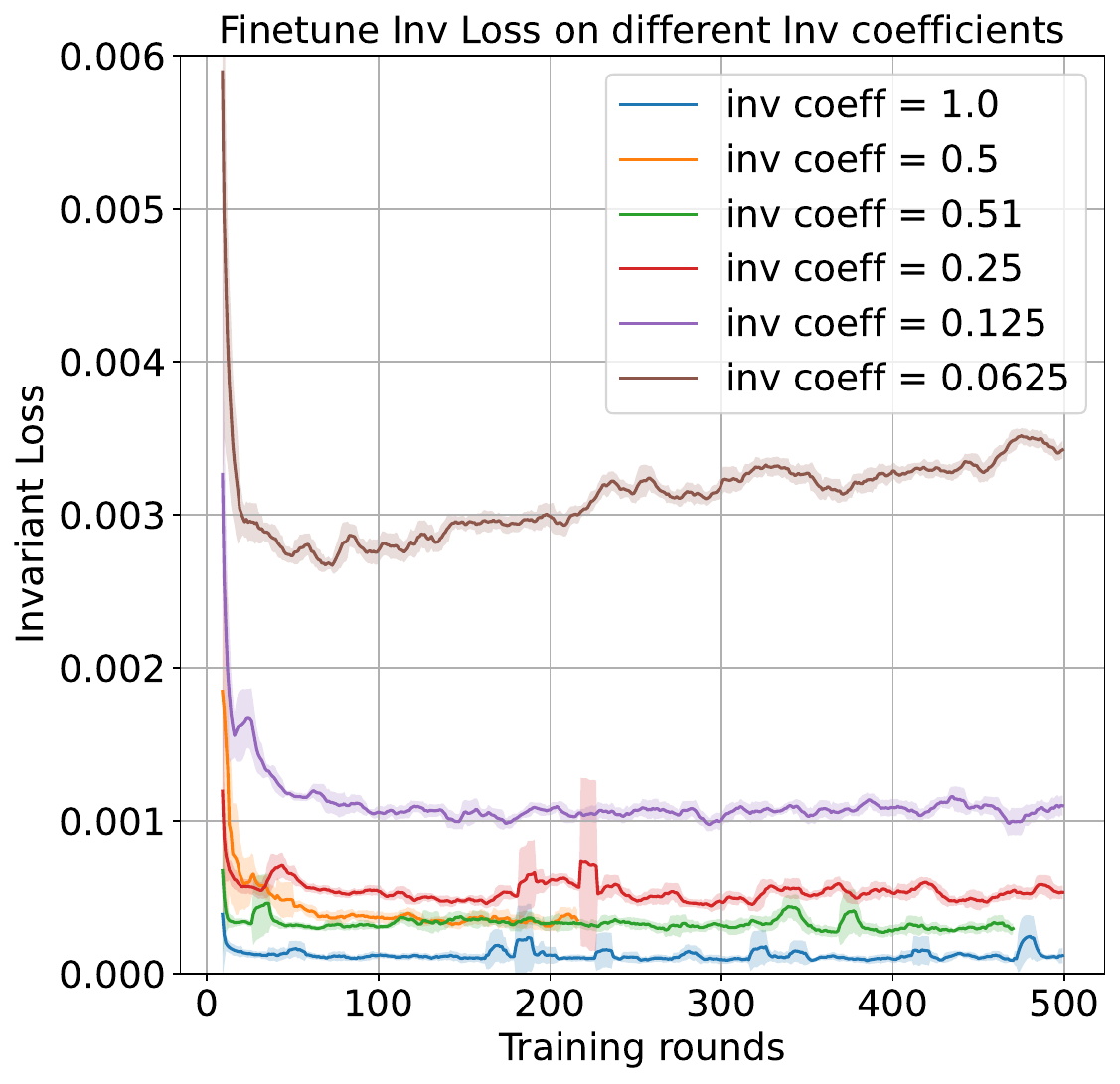} \\
\includegraphics[width = 0.47\linewidth]{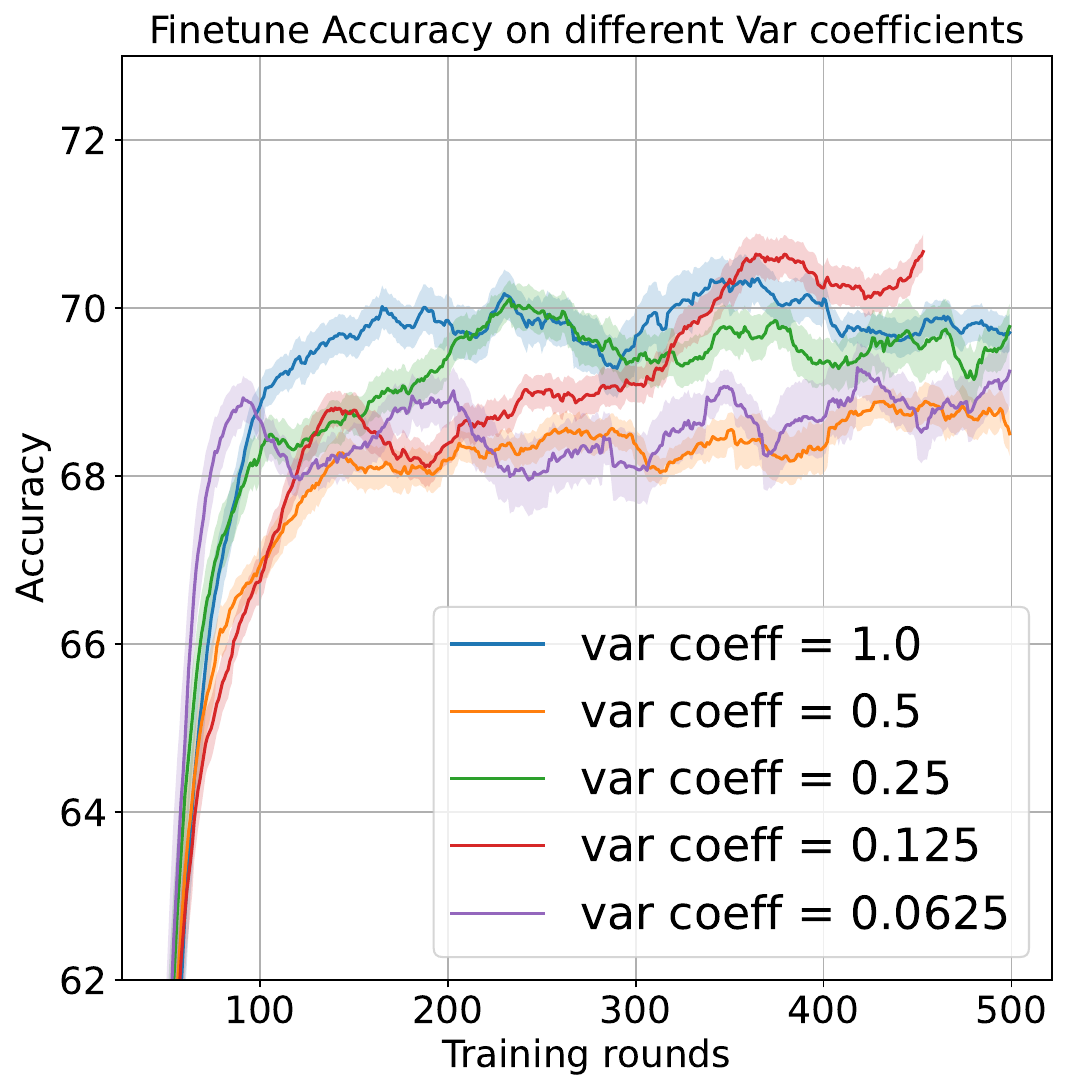} 
\includegraphics[width = 0.50\linewidth]{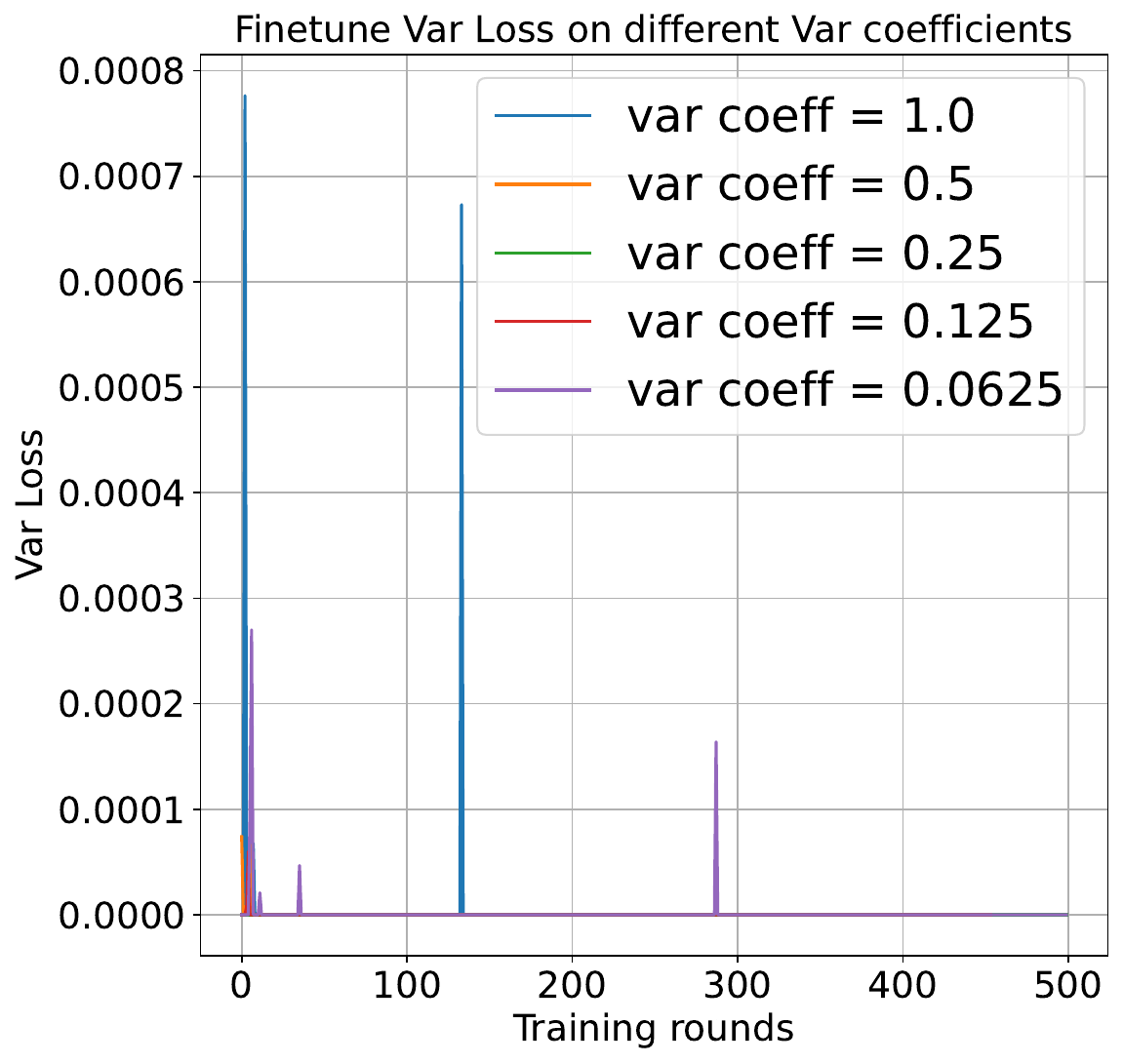} 
\caption{Training of UKIE under different learning coefficient settings.}
\label{fig:learning-coeff}
\end{figure}
First, we assess the impact of varying the label classification loss coefficient ($\alpha_\textrm{gtc}$) on the UKIE training performance. Fig.~\ref{fig:learning-coeff} shows that as $\alpha_\textrm{gtc}$ increases, the ability of $z^{(i)}_{V}$ to extract significant and meaningful information also increases. However, beyond a certain point, specifically when $\alpha_\textrm{gtc}\geq 1$, the accuracy improvement plateaus, showing little change. Consequently, we set $\alpha_\textrm{gtc}=1$ as the optimal coefficient for further adjustments.

Secondly, we assess the performance of the UKIE reconstruction property by varying the reconstruction coefficient $\alpha_\textrm{rec}$. The figure illustrates that as $\alpha_\textrm{gtc}$ increases, the ability of the invariant to extract significant and meaningful information also increases. However, beyond a certain point, specifically when $\alpha_\textrm{gtc}\geq 1$, the accuracy improvement plateaus, showing little change. Consequently, we set $\alpha_\textrm{gtc}=1$ as the optimal value of the coefficient for further adjustments to $\alpha_\textrm{rec}$.

Thirdly, we evaluate the effectiveness of UKIE invariant representation learning by adjusting the invariant coefficient $\alpha_\textrm{iv}$. As depicted in the figure, the quality of the invariant representation, as indicated by the invariant loss and accuracy in predicting invariant classes, diminishes as $\alpha_\textrm{iv}$ decreases. Nevertheless, we recognize that the quality variation becomes negligible when $\alpha_\textrm{iv}$ is greater than or equal to 0.25. Consequently, we select $\alpha_\textrm{iv}=0.25$ as the optimal configuration.

When evaluating the variant coefficient $\alpha_\textrm{v}$, we recognize that variant representations can adapt remarkably effectively. Consequently, we understand that setting a low value for $\alpha_\textrm{v}$ can still yield excellent performance in the UKIE context, while also preserving the rapid adaptation capability for other tasks.

\subsubsection{Assessing the domain generalization capabilities of UKIE}
We assess the performance of UKIE in extracting the invariant representations and predicting the invariant representations. By doing so, we can validate the robustness of causality invariant representations against the domain generalization gap. To do so, we leverage the DomainBed library. Specifically, we train three different UKIE models on Colored-MNIST \cite{2020-DG-DomainBed} with three different domains. We apply the evaluation to all other domains. Table~\ref{tab:ukie-domain} reveals the assessment of UKIE on different domains. 

The results from Table~\ref{tab:ukie-domain} reveal two primary criteria. Firstly, the invariant representations $z_K$ bring a lot of causal information, which leads to the prediction of the source data's ground truth. This is aligned with the SCM. Secondly, the data invariant representations are consistent among different data domains. Therefore, we can assume that we can apply the sparse update of semantic knowledge among distributed devices as mentioned in Section~\ref{sec:semantic-channel} without significant divergence among devices' semantic knowledge.

\section{Conclusion}
\label{sec:conclusion}
This paper introduces a novel algorithm for extracting semantic knowledge, i.e., Unified Knowledge retrieval via Invariant Extractor, without manual human intervention, utilizing a structured causal model. It identifies two key aspects of data: invariant and variant components, with the former being crucial for accurate predictions. The data extracted from UKIE remain consistent with the data's ground truth and do not need frequent updates, making it ideal for knowledge-aided semantic channels. Various numerical experiments demonstrate that this invariant knowledge significantly aids in data reconstruction while reducing the information load on the physical channel.


\balance
\bibliographystyle{IEEEtran}
\bibliography{UKIE.bib}

\begin{thebibliography}{10}
\providecommand{\url}[1]{#1}
\csname url@samestyle\endcsname
\providecommand{\newblock}{\relax}
\providecommand{\bibinfo}[2]{#2}
\providecommand{\BIBentrySTDinterwordspacing}{\spaceskip=0pt\relax}
\providecommand{\BIBentryALTinterwordstretchfactor}{4}
\providecommand{\BIBentryALTinterwordspacing}{\spaceskip=\fontdimen2\font plus
\BIBentryALTinterwordstretchfactor\fontdimen3\font minus
  \fontdimen4\font\relax}
\providecommand{\BIBforeignlanguage}[2]{{%
\expandafter\ifx\csname l@#1\endcsname\relax
\typeout{** WARNING: IEEEtran.bst: No hyphenation pattern has been}%
\typeout{** loaded for the language `#1'. Using the pattern for}%
\typeout{** the default language instead.}%
\else
\language=\csname l@#1\endcsname
\fi
#2}}
\providecommand{\BIBdecl}{\relax}
\BIBdecl

\bibitem{sun2025edge}
Y.~Sun, Y.~Liu, S.~Guo, X.~Qiu, J.~Chen, J.~Hao, and D.~Niyato, ``Edge large ai
  model agent-empowered cognitive multimodal semantic communication,''
  \emph{IEEE Transactions on Mobile Computing}, no.~01, pp. 1--18, 2025.

\bibitem{liu2025multi}
Z.~Liu and T.~Ratnarajah, ``Multi-user semantic communication for interactive
  speech dialogue with dynamic resource allocation,'' \emph{IEEE Transactions
  on Network Science and Engineering}, 2025.

\bibitem{2024-SemCom-Survey2}
T.-H. Vu, S.~K. Jagatheesaperumal, M.-D. Nguyen, N.~V. Huynh, S.~Kim, and Q.-V.
  Pham, ``{Applications of Generative AI (GAI) for Mobile and Wireless
  Networking: A Survey},'' \emph{IEEE Internet of Things Journ.}, Aug. 2024.

\bibitem{2023-FL-HCFL}
M.-D. Nguyen, S.-M. Lee, Q.-V. Pham, D.~T. Hoang, D.~N. Nguyen, and W.-J.
  Hwang, ``{HCFL: A High Compression Approach for Communication-Efficient
  Federated Learning in Very Large Scale IoT Networks},'' \emph{IEEE Trans. on
  Mob. Comp.}, May 2023.

\bibitem{2024-SemCom-DRGO}
M.-D. Nguyen, Q.~V. Do, Z.~Yang, W.-J. Hwang, and Q.-V. Pham, ``Distortion
  resilience for goal-oriented semantic communication,'' \emph{IEEE
  Transactions on Mobile Computing}, 2025.

\bibitem{yu2025multi}
X.~Yu, T.~Lv, W.~Li, W.~Ni, D.~Niyato, and E.~Hossain, ``Multi-task semantic
  communication with graph attention-based feature correlation extraction,''
  \emph{IEEE Transactions on Mobile Computing}, 2025.

\bibitem{2021-SEM-DeepSC}
H.~Xie, Z.~Qin, G.~Y. Li, and B.-H. Juang, ``Deep learning enabled semantic
  communication systems,'' \emph{IEEE Trans. Sign. Process.}, Sep. 2021.

\bibitem{2023-SemCom-UDeepSC}
G.~Zhang, Q.~Hu, Z.~Qin, Y.~Cai, G.~Yu, and X.~Tao, ``A unified multi-task
  semantic communication system for multimodal data,'' \emph{IEEE Transactions
  on Communications}, vol.~72, no.~7, pp. 4101--4116, 2024.

\bibitem{2022-SemCom-MUDeepSC}
H.~Xie, Z.~Qin, X.~Tao, and K.~B. Letaief, ``{Task-Oriented Multi-User Semantic
  Communications},'' \emph{IEEE Jour. of Sel. Areas in Comm.}, Mar. 2022.

\bibitem{2023-SemCom-MemDeepSC}
H.~Xie, Z.~Qin, and G.~Y. Li, ``Semantic communication with memory,''
  \emph{IEEE Jour. of Sel. Areas in Comm.}, Jun. 2023.

\bibitem{2022-SEM-AdaptableSemanticCompression}
C.~Liu, C.~Guo, Y.~Yang, and N.~Jiang, ``{Adaptable Semantic Compression and
  Resource Allocation for Task-Oriented Communications},'' \emph{IEEE Trans.
  Cogn. Comm. and Netw.}, Jun. 2024.

\bibitem{2020-SemCom-DJSCCF}
D.~B. Kurka and D.~Gündüz, ``{DeepJSCC-f: Deep Joint Source-Channel Coding of
  Images With Feedback},'' \emph{IEEE Jour. of Sel. Areas in Comm.}, May 2020.

\bibitem{2019-SemCOm-DJSCC-WIT}
E.~Bourtsoulatze, D.~Burth~Kurka, and D.~Gündüz, ``{Deep Joint Source-Channel
  Coding for Wireless Image Transmission},'' \emph{IEEE Trans. Cogn. Comm. and
  Netw.}, Apr. 2019.

\bibitem{2024-SemCom-AdaSem}
Q.~Liao and T.-Y. Tung, ``{AdaSem: Adaptive Goal-Oriented Semantic
  Communications for End-to-End Camera Relocalization},'' in \emph{INFOCOM},
  May 2024.

\bibitem{2024-SemCom-SemCC}
S.~Tang, Q.~Yang, L.~Fan, X.~Lei, A.~Nallanathan, and G.~K. Karagiannidis,
  ``Contrastive learning-based semantic communications,'' \emph{IEEE Trans. on
  Comm.}, Oct. 2024.

\bibitem{2024-SemCom-DeepMA}
W.~Zhang, K.~Bai, S.~Zeadally, H.~Zhang, H.~Shao, H.~Ma, and V.~C.~M. Leung,
  ``{DeepMA}: End-to-end deep multiple access for wireless image transmission
  in semantic communication,'' \emph{IEEE Trans. Cogn. Comm. and Netw.}, Apr.
  2024.

\bibitem{2023-SemCom-GenerativeJSCC}
E.~Erdemir, T.-Y. Tung, P.~L. Dragotti, and D.~Gündüz, ``Generative joint
  source-channel coding for semantic image transmission,'' \emph{IEEE Jour. of
  Sel. Areas in Comm.}, Aug. 2023.

\bibitem{2024-SemCom-SCAN}
G.~Zhang, Q.~Hu, Y.~Cai, and G.~Yu, ``Scan: Semantic communication with
  adaptive channel feedback,'' \emph{IEEE Transactions on Cognitive
  Communications and Networking}, Apr. 2024.

\bibitem{2023-SemCom-DeepJSCCV}
W.~Zhang, H.~Zhang, H.~Ma, H.~Shao, N.~Wang, and V.~C.~M. Leung, ``Predictive
  and adaptive deep coding for wireless image transmission in semantic
  communication,'' \emph{IEEE Trans. on Wirel. Comm.}, Jan. 2023.

\bibitem{2023-SemCom-VLSCC}
B.~Zhang, Z.~Qin, and G.~Y. Li, ``Semantic communications with variable-length
  coding for extended reality,'' \emph{IEEE Jour. of Sel. Topics in Sign.
  Proc.}, Sep. 2023.

\bibitem{2023-SemCom-TaskUnaware}
H.~Zhang, S.~Shao, M.~Tao, X.~Bi, and K.~B. Letaief, ``{Deep Learning-Enabled
  Semantic Communication Systems With Task-Unaware Transmitter and Dynamic
  Data},'' \emph{IEEE Jour. of Sel. Areas in Comm.}, Apr. 2023.

\bibitem{2024-SemCom-JCM}
Y.~Bo, Y.~Duan, S.~Shao, and M.~Tao, ``Joint coding-modulation for digital
  semantic communications via variational autoencoder,'' \emph{IEEE Trans. on
  Comm.}, Sep. 2024.

\bibitem{2022-KG-LogicalQueryAnswering}
M.~Galkin, Z.~Zhu, H.~Ren, and J.~Tang, ``{Inductive Logical Query Answering in
  Knowledge Graphs},'' in \emph{Adv. Neural Inform. Process. Syst.}, Dec. 2022.

\bibitem{2022-KG-RelationExtraction}
N.~Ding, X.~Wang, Y.~Fu, G.~Xu, R.~Wang, P.~Xie, Y.~Shen, F.~Huang, H.-T.
  Zheng, and R.~Zhang, ``{Prototypical Representation Learning for Relation
  Extraction},'' in \emph{Int. Conf. Learn. Represent.}, May 2021.

\bibitem{2022-SemCom-AdaptiveBitRate}
Q.~Zhou, R.~Li, Z.~Zhao, Y.~Xiao, and H.~Zhang, ``{Adaptive Bit Rate Control in
  Semantic Communication With Incremental Knowledge-Based HARQ},'' \emph{IEEE
  Open Journal of the Communications Society}, Mar. 2022.

\bibitem{2023-IKR-T2IPR}
D.~Jiang and M.~Ye, ``{Cross-Modal Implicit Relation Reasoning and Aligning for
  Text-to-Image Person Retrieval},'' in \emph{IEEE Conf. Comput. Vis. Pattern
  Recog.}, Apr. 2023.

\bibitem{2023-IKR-CIRCRN}
Q.~Yang, M.~Ye, Z.~Cai, K.~Su, and B.~Du, ``{Composed Image Retrieval via Cross
  Relation Network With Hierarchical Aggregation Transformer},'' \emph{IEEE
  Trans. Image Process.}, Apr. 2023.

\bibitem{2023-IKR-Reveal}
Z.~Hu, A.~Iscen, C.~Sun, Z.~Wang, K.~Chang, Y.~Sun, C.~Schmid, D.~A. Ross, and
  A.~Fathi, ``{REVEAL:} retrieval-augmented visual-language pre-training with
  multi-source multimodal knowledge memory,'' in \emph{IEEE Conf. Comput. Vis.
  Pattern Recog.}, Apr. 2023.

\bibitem{2022-DG-CIRL}
F.~Lv, J.~Liang, S.~Li, B.~Zang, C.~H. Liu, Z.~Wang, and D.~Liu, ``{Causality
  Inspired Representation Learning for Domain Generalization},'' in \emph{IEEE
  Conf. Comput. Vis. Pattern Recog.}, May 2022.

\bibitem{2022-DG-KLGuided}
A.~T. Nguyen, T.~Tran, Y.~Gal, P.~Torr, and A.~G. Baydin, ``{KL} guided domain
  adaptation,'' in \emph{Int. Conf. Learn. Represent.}, Dec. 2022.

\bibitem{2021-DG-IRL-DDT}
A.~T. Nguyen, T.~Tran, Y.~Gal, and A.~G. Baydin, ``{Domain Invariant
  Representation Learning with Domain Density Transformations},'' in \emph{Adv.
  Neural Inform. Process. Syst.}, Dec. 2021.

\bibitem{2022-DG-InvariantRationale}
Y.~Wu, X.~Wang, A.~Zhang, X.~He, and T.-S. Chua, ``{Discovering Invariant
  Rationales for Graph Neural Networks},'' in \emph{Int. Conf. Learn.
  Represent.}, May 2022.

\bibitem{2020-DG-EntropyReg}
S.~Zhao, M.~Gong, T.~Liu, H.~Fu, and D.~Tao, ``{Domain Generalization via
  Entropy Regularization},'' in \emph{Adv. Neural Inform. Process. Syst.}, Dec.
  2020.

\bibitem{2018-DG-CIAN}
Y.~Li, X.~Tian, M.~Gong, Y.~Liu, T.~Liu, K.~Zhang, and D.~Tao, ``{Deep Domain
  Generalization via Conditional Invariant Adversarial Networks},'' in
  \emph{Eur. Conf. Comput. Vis.}, Sep. 2018.

\bibitem{2021-DG-ExploitDomainSpecific}
H.~M. Bui, T.~Tran, A.~T. Tran, and D.~Phung, ``E{xploiting Domain-Specific
  Features to Enhance Domain Generalization},'' in \emph{Adv. Neural Inform.
  Process. Syst.}, Dec. 2021.

\bibitem{2014-ML-GAN}
I.~Goodfellow, J.~Pouget-Abadie, M.~Mirza, B.~Xu, D.~Warde-Farley, S.~Ozair,
  A.~Courville, and Y.~Bengio, ``{Generative Adversarial Networks},'' in
  \emph{Adv. Neural Inform. Process. Syst.}, Dec. 2014.

\bibitem{2022-SemCom-KG-Cognitive}
F.~Zhou, Y.~Li, X.~Zhang, Q.~Wu, X.~Lei, and R.~Q. Hu, ``Cognitive semantic
  communication systems driven by knowledge graph,'' in \emph{IEEE Int. Conf.
  on Comm.}, 2022.

\bibitem{2015-DL-CVAE}
K.~Sohn, X.~Yan, and H.~Lee, ``Learning structured output representation using
  deep conditional generative models,'' in \emph{Adv. Neural Inform. Process.
  Syst.}, Dec. 2015.

\bibitem{CVAE-3DHuman}
Y.~Cai, Y.~Wang, Y.~Zhu, T.-J. Cham, J.~Cai, J.~Yuan, J.~Liu, C.~Zheng, S.~Yan,
  H.~Ding, X.~Shen, D.~Liu, and N.~M. Thalmann, ``A unified 3d human motion
  synthesis model via conditional variational auto-encoder,'' in \emph{Int.
  Conf. Comput. Vis.}, Oct. 2021.

\bibitem{CVAE-Audio2Gestures}
J.~Li, D.~Kang, W.~Pei, X.~Zhe, Y.~Zhang, Z.~He, and L.~Bao, ``Audio2gestures:
  Generating diverse gestures from speech audio with conditional variational
  autoencoders,'' in \emph{Int. Conf. Comput. Vis.}, Oct. 2021.

\bibitem{CVAE-ManifoldDimension}
Y.~Zheng, T.~He, Y.~Qiu, and D.~Wipf, ``Learning manifold dimensions with
  conditional variational autoencoders,'' in \emph{Adv. Neural Inform. Process.
  Syst.}, Dec. 2022.

\bibitem{CGAN-I2I}
J.~Lin, Y.~Xia, T.~Qin, Z.~Chen, and T.-Y. Liu, ``Conditional image-to-image
  translation,'' in \emph{IEEE Conf. Comput. Vis. Pattern Recog.}, Jun. 2018.

\bibitem{CGAN-ImageGeneration}
X.~Ding, Y.~Wang, Z.~Xu, W.~J. Welch, and Z.~J. Wang, ``Cc{\{}gan{\}}:
  Continuous conditional generative adversarial networks for image
  generation,'' in \emph{Int. Conf. Learn. Represent.}, May 2021.

\bibitem{CGAN-ImageSynthesis}
R.~Liu, Y.~Ge, C.~L. Choi, X.~Wang, and H.~Li, ``Divco: Diverse conditional
  image synthesis via contrastive generative adversarial network,'' in
  \emph{IEEE Conf. Comput. Vis. Pattern Recog.}, Jul. 2021.

\bibitem{CDM-ImageRecon}
M.~Zhussip, I.~S. Koshelev, and S.~Lefkimmiatis, ``A modular conditional
  diffusion framework for image reconstruction,'' in \emph{Adv. Neural Inform.
  Process. Syst.}, Dec. 2024.

\bibitem{CDM-VideoSynthesis}
P.~Esser, J.~Chiu, P.~Atighehchian, J.~Granskog, and A.~Germanidis, ``Structure
  and content-guided video synthesis with diffusion models,'' in \emph{Int.
  Conf. Comput. Vis.}, Jul. 2023.

\bibitem{snell2017prototypical}
J.~Snell, K.~Swersky, and R.~Zemel, ``Prototypical networks for few-shot
  learning,'' \emph{Advances in neural information processing systems},
  vol.~30, 2017.

\bibitem{2023-DG-FDG1}
R.~Zhang, Q.~Xu, J.~Yao, Y.~Zhang, Q.~Tian, and Y.~Wang, ``{Federated Domain
  Generalization With Generalization Adjustment},'' in \emph{IEEE Conf. Comput.
  Vis. Pattern Recog.}, Jun. 2023.

\bibitem{2023-FL-FDG2}
L.~Zhang, X.~Lei, Y.~Shi, H.~Huang, and C.~Chen, ``{Federated Learning for IoT
  Devices With Domain Generalization},'' \emph{IEEE Internet of Things
  Journal}, May 2023.

\bibitem{2022-IL-VICReg}
A.~Bardes, J.~Ponce, and Y.~LeCun, ``{VICR}eg: Variance-invariance-covariance
  regularization for self-supervised learning,'' in \emph{Int. Conf. Learn.
  Represent.}, May 2022.

\bibitem{2016-Causal-Primer}
J.~Pearl, M.~Glymour, and N.~Jewell, \emph{{Causal Inference in Statistics: A
  Primer}}.\hskip 1em plus 0.5em minus 0.4em\relax Wiley, 2016.

\bibitem{2013-DL-VAE}
D.~P. Kingma and M.~Welling, ``{Auto-Encoding Variational Bayes},'' \emph{arXiv
  preprint arXiv:1312.6114}, Dec. 2013.

\bibitem{2022-IL-DIR}
Y.~Wu, X.~Wang, A.~Zhang, X.~He, and T.-S. Chua, ``{Discovering Invariant
  Rationales for Graph Neural Networks},'' in \emph{Int. Conf. Learn.
  Represent.}, May 2022.

\bibitem{2018-MeL-Reptile}
A.~Nichol, J.~Achiam, and J.~Schulman, ``{On First-Order Meta-Learning
  Algorithms},'' \emph{arXiv preprint arXiv:1803.02999}, Mar. 2018.

\bibitem{2017-Mel-MAML}
C.~Finn, P.~Abbeel, and S.~Levine, ``{Model-Agnostic Meta-Learning for Fast
  Adaptation of Deep Networks},'' in \emph{Int. Conf. Mach. Learn.}, Jul. 2017.

\bibitem{2010-Data-MNIST}
Y.~LeCun and C.~Cortes, ``The {MNIST} database of handwritten digits,'' 2010.

\bibitem{2017-Data-EMNIST}
G.~Cohen, S.~Afshar, J.~Tapson, and A.~Van~Schaik, ``{EMNIST}: Extending
  {MNIST} to handwritten letters,'' in \emph{International Joint Conference on
  Neural Networks}, Mar. 2017.

\bibitem{2009-Data-CIFAR}
A.~Krizhevsky, G.~Hinton \emph{et~al.}, ``{Learning Multiple Layers of Features
  from Tiny Images},'' \emph{University of Toronto}, 2009.

\bibitem{2018-Data-CINIC10}
L.~N. Darlow, E.~J. Crowley, A.~Antoniou, and A.~J. Storkey, ``{CINIC-10 is not
  ImageNet or CIFAR-10},'' \emph{arXiv preprint arXiv:1810.03505}, 2018.

\bibitem{2016-DL-Resnet}
K.~He, X.~Zhang, S.~Ren, and J.~Sun, ``{Deep Residual Learning for Image
  Recognition},'' in \emph{IEEE Conf. Comput. Vis. Pattern Recog.}, Jul. 2016.

\bibitem{2021-ViT-VisionTransformer}
A.~Dosovitskiy, L.~Beyer, A.~Kolesnikov, D.~Weissenborn, X.~Zhai,
  T.~Unterthiner, M.~Dehghani, M.~Minderer, G.~Heigold, S.~Gelly, J.~Uszkoreit,
  and N.~Houlsby, ``An image is worth 16x16 words: Transformers for image
  recognition at scale,'' in \emph{Int. Conf. Learn. Represent.}, May 2021.

\bibitem{2018-DL-DisentanglingFactorising}
H.~Kim and A.~Mnih, ``Disentangling by factorising,'' in \emph{Int. Conf. Mach.
  Learn.}, Jul. 2018.

\bibitem{dipave}
A.~Kumar, P.~Sattigeri, and A.~Balakrishnan, ``{Variational Inference of
  disentangled latent concepts from unlabeled observations},'' in \emph{Int.
  Conf. Learn. Represent.}, May 2018.

\bibitem{2016-DL-BetaVAE}
I.~Higgins, L.~Matthey, A.~Pal, C.~Burgess, X.~Glorot, M.~Botvinick,
  S.~Mohamed, and A.~Lerchner, ``beta-{VAE}: Learning basic visual concepts
  with a constrained variational framework,'' in \emph{Int. Conf. Learn.
  Represent.}, May 2017.

\bibitem{2016-GAN-InfoGan}
X.~Chen, Y.~Duan, R.~Houthooft, J.~Schulman, I.~Sutskever, and P.~Abbeel,
  ``Infogan: Interpretable representation learning by information maximizing
  generative adversarial nets,'' in \emph{Adv. Neural Inform. Process. Syst.},
  May 2016.

\bibitem{2020-GAN-UnGan}
A.~Voynov and A.~Babenko, ``Unsupervised discovery of interpretable directions
  in the gan latent space,'' in \emph{Int. Conf. Mach. Learn.}, Jul. 2020.

\bibitem{2021-GAN-IPGAN}
H.~Yang, L.~Chai, Q.~Wen, S.~Zhao, Z.~Sun, and S.~He, ``Discovering
  interpretable latent space directions of gans beyond binary attributes,'' in
  \emph{IEEE Conf. Comput. Vis. Pattern Recog.}, Aug. 2021.

\bibitem{2021-DL-LieDisentanglement}
X.~Zhu, C.~Xu, and D.~Tao, ``{Commutative Lie Group VAE for Disentanglement
  Learning},'' in \emph{Int. Conf. Mach. Learn.}, Jul. 2021.

\bibitem{2020-DG-DomainBed}
I.~Gulrajani and D.~Lopez-Paz, ``{In Search of Lost Domain Generalization},''
  in \emph{Int. Conf. Learn. Represent.}, May 2021.

\end{thebibliography}
\end{document}